%% file: main.tex
\documentclass{article}
\usepackage{ijcai21}

\usepackage{times}
\usepackage{soul}
\usepackage{url}
\usepackage[hidelinks]{hyperref}
\usepackage[utf8]{inputenc}
\usepackage[small]{caption}
\usepackage{graphicx}
\usepackage{amsmath}
\usepackage{amsthm}
\usepackage{booktabs}
\usepackage{algorithm}
\usepackage{algorithmic}
\usepackage{nicefrac}
\usepackage[capitalize,noabbrev]{cleveref}

\usepackage{amssymb}

\newtheorem{theorem}{Theorem}
\newtheorem{lemma}{Lemma}
\newtheorem{assumption}{Assumption}

\DeclareMathOperator*{\argmin}{arg\,min}

\newcommand{\eqdef}{\stackrel{\text{def}}{=}}
\newcommand{\bbE}{\mathbb{E}}

\newcommand{\bbV}{\mathbb{V}}
\newcommand{\calF}{\mathcal{F}}

\newcommand{\wt}{\widetilde}

\newcommand{\tO}{\wt O}

\newcommand{\ind}{\mathbb{I}}
\newcommand{\indevent}[1]{\ind \{ #1 \}}

\newcommand{\KL}[2]{\text{KL}(#1 \;\|\; #2)}

\newcommand{\regret}{R_K}
\newcommand{\tregret}[1]{\wt R_{#1}}

\newcommand{\sinit}{s_0}
\newcommand{\ssink}{g}
\newcommand{\policyset}{\Pi}
\newcommand{\proppolicyset}{\policyset_\text{proper}}

\newcommand{\ocset}{\Delta}

\newcommand{\boundedocset}[1]{\ocset( #1 )}
\newcommand{\wtboundedocset}[2]{\widetilde{\ocset}_{ #2 } ( #1 )}
\newcommand{\fastpolicy}{\pi^f}
\newcommand{\optimisticfastpolicy}[1]{\wt \pi^f_{ #1 }}

\newcommand{\totaltime}{T}
\newcommand{\Hm}{H^m}
\newcommand{\Ik}{I^k}

\newcommand{\ctg}[1]{J^{#1}}

\newcommand{\optimisticctg}[1]{\wt{J}^{#1}}
\newcommand{\policytime}[1]{T^{#1}}
\newcommand{\optimisticpolicytime}[1]{\wt{T}^{#1}}
\newcommand{\geventi}[1]{\Omega^{#1}}

\newcommand{\indgeventi}[1]{\ind \{ \geventi{#1} \}}
\newcommand{\numintervals}{M}
\newcommand{\cmin}{c_{\text{min}}}
\newcommand{\traj}[1]{U^{#1}}
\newcommand{\trajconcat}[1]{\bar U^{#1}}
\newcommand{\optimisticpi}[1]{\tilde \pi_{#1}}
\newcommand{\optimisticP}[1]{\wt P_{#1}}

\newcommand{\numvisitsuntilknownbern}{\frac{D |S|}{\cmin^2} \log \frac{D |S| |A|}{\delta \cmin}}

\title{Stochastic Shortest Path with Adversarially Changing Costs}

\author{
Aviv Rosenberg$^1$\And
Yishay Mansour$^{1,2}$\\
\affiliations
$^1$Tel Aviv University, Israel\\
$^2$Google Research, Tel Aviv\\
\emails
\{avivros007,mansour.yishay\}@gmail.com
}

\begin{document}

\maketitle

\begin{abstract}
    Stochastic shortest path (SSP) is a well-known problem in planning and control, in which an agent has to reach a goal state in minimum total expected cost.
    In this paper we present the adversarial SSP model that also accounts for adversarial changes in the costs over time, while the underlying transition function remains unchanged.
    Formally, an agent interacts with an SSP environment for K episodes, the cost function changes arbitrarily between episodes, and the transitions are unknown to the agent.
    We develop the first algorithms for adversarial SSPs and prove high probability regret bounds of square-root K assuming all costs are strictly positive, and sub-linear regret in the general case.
    We are the first to consider this natural setting of adversarial SSP and obtain sub-linear regret for it.
\end{abstract}

\section{Introduction}

Stochastic shortest path (SSP) is one of the most basic models in reinforcement learning (RL). 
It features an agent that interacts with a Markov decision process (MDP) with the aim of reaching a predefined goal state in minimum total expected cost.
Many important RL problems fall into the SSP framework, e.g., car navigation and Atari games, and yet it was only rarely studied from a theoretical point of view until very recently, mainly due to its challenging nature in comparison to finite-horizon, average-reward or discounted MDPs. 
For example, in SSP some polices might suffer infinite cost.

An important aspect that the standard SSP model fails to capture is changes in the environment over time (e.g., changes in traffic when navigating a car).
In the finite-horizon setting, the adversarial MDP model was proposed to address changing environments, and has gained considerable popularity in recent years.
It allows the cost function to change arbitrarily over time, while still assuming a fixed transition function.

In this work we present the adversarial SSP model that introduces adversarially changing costs to the classical SSP model.
Formally, the agent interacts with an SSP instance for $K$ episodes, and the cost function changes arbitrarily between episodes.
The agent's objective is to reach the goal state in all episodes while minimizing its total expected cost. 
Its performance is measured by the \emph{regret}, defined as the cumulative difference between the agent's total cost in $K$ episodes and the expected total cost of the best policy in hindsight.

Finite-horizon MDPs are a special case of the general SSP problem where the agent is guaranteed to reach the goal state within a fixed number of steps $H$.
This model is extensively studied in recent years for both stochastic and adversarial costs.
In the adversarial MDP literature it is better known as the \emph{loop-free} SSP model.
While having a similar name, loop-free SSP follows the restrictive assumption that after $H$ steps the goal will be reached and is thus far less challenging.

As pointed out by \cite{tarbouriech2019noregret}, in the general SSP problem we face new challenges that do not arise in the loop-free version. 
Notably, it features two possibly conflicting objectives -- reaching the goal vs minimizing cost; and it requires handling unbounded value functions and episode lengths.
In the adversarial SSP model, these difficulties are further amplified as the adversary might encourage the learner to use ``slow'' policies and then punish her with large costs.

In this paper we propose the first algorithms for regret minimization in adversarial SSPs without any restrictive assumptions (namely, loop-free assumption).
While we leverage algorithmic and technical tools from both SSP and finite-horizon adversarial MDP, tackling the general SSP problem in the presence of an adversary requires novel techniques and careful analysis.
 Our algorithms are based on the popular online mirror descent (OMD) framework for online convex optimization (OCO).
However, naive application of OMD to SSP cannot overcome the challenges mentioned above as we later show, and we use carefully designed mechanisms to establish our theoretical guarantees.

The main contributions of this paper are as follows.
First, we formalize the adversarial SSP model and define the notion of learning and regret.
Second, we establish an efficient implementation of OMD in the SSP model with known transitions and study the conditions under which it guarantees near-optimal $\sqrt{K}$ expected regret, showing that some modifications are necessary.
Then, we illustrate the challenge of obtaining regret bounds in high probability in adversarial SSPs, and present a novel method that allows OMD to obtain its regret with high probability.
Finally, we tackle unknown transitions.
We describe the crucial adaptations that allow OMD to be combined with optimistic estimates of the transition function and guarantee $\sqrt{K}$ regret when all costs are strictly positive, and $K^{3/4}$ regret in the general case.
Hopefully, the infrastructure created in this paper for handling adversarial costs in SSPs with unknown transition function paves the way for future work to achieve minimax optimal regret bounds.

\paragraph{Related work.}
Early work by \cite{bertsekas1991analysis} studied the planning problem in SSPs, i.e., computing the optimal strategy efficiently when parameters are known. 
Under certain assumptions, they established that the optimal strategy is a deterministic stationary policy (a mapping from states to actions) and can be
computed efficiently using standard planning algorithms, e.g., Value Iteration and LP.

Recently \cite{tarbouriech2019noregret} presented the problem of learning SSPs (with stochastic costs) and provided the first algorithms with sub-linear regret but with dependence on the minimal cost $\cmin$.
Their results were further improved by \cite{cohen2020ssp} that eliminate the $\cmin$ dependence and prove high probability regret bound of $\tO(D |S| \sqrt{|A| K})$ complemented by a nearly matching lower bound of $\Omega (D \sqrt{|S| |A| K})$, where $D$ is the diameter, $S$ is the state space and $A$ is the action space.

As mentioned before, regret minimization in RL is extensively studied in recent years, but the literature mainly focuses on the average-reward infinite-horizon model \cite{bartlett2009regal,AuerUCRL} and on the finite-horizon model \cite{osband2016generalization,azar2017minimax,dann2017unifying,jin2018q,zanette2019tighter,efroni2019tight}.
Adversarial MDPs were also first studied in the average-reward model \cite{even2009online,neu2014bandit}, before focusing on the finite-horizon setting which is typically referred to as loop-free SSP.
Early work in this setting by \cite{neu2010ossp} used a reduction to multi-arm bandit \cite{auer2002nonstochastic}, but then \cite{zimin2013online} introduced the O-REPS framework, which is the implementation of OMD in finite-horizon MDPs.
All these works assume known transition function, but more recently \cite{neu2012unknown,rosenberg2019full,rosenberg2019bandit,jin2019learning,efroni2020optimistic,cai2019provably} consider unknown transitions.

We stress that all previous work in the adversarial setting made the restrictive loop-free assumption, avoiding the main challenges tackled in this paper.
Building on our methodologies, \cite{chen2020minimax} recently extended our work and obtained minimax optimal $\sqrt{K}$ regret with known transitions.
However, they do not consider the more challenging unknown transitions case, and also assume that the learner knows in advance the running time of the best policy in hindsight.

\section{Preliminaries}

An adversarial SSP problem is defined by an MDP $M = (S,A,P,\sinit,\ssink)$ and a sequence $\{ c_k: S \times A \rightarrow [0,1] \}_{k=1}^K$ of cost functions. 
$S$ and $A$ are finite state and action spaces, respectively, $\sinit \in S$ is an initial state and $\ssink \not\in S$ is the goal state.
$P$ is a transition function such that $P(s' \mid s,a)$ gives the probability to move to $s'$ when taking action $a$ in state $s$, and thus $\sum_{s' \in S \cup \{ \ssink \}} P(s' \mid s,a) = 1$ for every $(s,a) \in S \times A$.

The learner interacts with $M$ in episodes, where $c_k$ is the cost function for episode $k$.
However, it is revealed to the learner only in the end of the episode.
Formally, the learner starts each episode $k$ at the initial state\footnote{Our algorithms readily extend to a fixed initial distribution.} $s_1^k = \sinit$.
In each step $i$ of the episode, the learner observes its current state $s_i^k$, picks an action $a_i^k$ and moves to the next state $s_{i+1}^k$ sampled from $P(\cdot \mid s_i^k,a_i^k)$.
The episode ends when the goal state $\ssink$ is reached, and then the learner observes $c_k$ and suffers cost $\sum_{i=1}^{\Ik} c_k(s_i^k,a_i^k)$ where $\Ik$ is the length of the episode.
Importantly, $\Ik$ is a random variable that might be infinite.
This is the unique challenge of SSP compared to finite-horizon.

\paragraph{Proper Policies.}
A stationary policy $\pi : A \times S \rightarrow [0,1]$ is a mapping such that $\pi(a \mid s)$ gives the probability that action $a$ is selected in state $s$.
A policy $\pi$ is called \emph{proper} if playing according to $\pi$ ensures that the goal state is reached with probability $1$ when starting from any state (otherwise it is \emph{improper}).
Since reaching the goal is one of the learner's main objectives, we make the basic assumption that there exists at least one proper policy.
This is equivalent to the assumption that the goal state is reachable from every state, which is clearly a necessary assumption.

We denote by $\policytime{\pi}(s)$ the expected hitting time of $\ssink$ when playing according to $\pi$ and starting at $s$.
In particular, if $\pi$ is proper then $\policytime{\pi}(s)$ is finite for all $s$, and if $\pi$ is improper there must exist some $s' \in S$ such that $\policytime{\pi}(s') = \infty$.
When paired with a cost function $c: S \times A \rightarrow [0,1]$, any policy $\pi$ induces a \emph{cost-to-go function} $\ctg{\pi} : S \rightarrow [0, \infty]$, where $\ctg{\pi}(s)$ is the expected cost when playing policy $\pi$ and starting at state $s$, i.e.,
$
    \ctg{\pi}(s) = \lim_{T \rightarrow \infty} \bbE \bigl[ \sum_{t=1}^T c(s_t,a_t) \mid P,\pi,s_1 = s \bigr]
$.
For a proper policy $\pi$, it follows that $\ctg{\pi}(s)$ is finite for all $s$. 

Under the additional assumption that every improper policy suffers infinite expected cost from some state, \cite{bertsekas1991analysis} show that the optimal policy is stationary, deterministic and proper; and that every proper policy $\pi$ satisfies the following Bellman equations for every $s \in S$:
\begin{align}
    \nonumber
    \ctg{\pi}(s)
    & =
    \sum_{a \in A} \pi(a \mid s) \Bigl( c(s,a) + \sum_{s' \in S} P(s' \mid s,a) \ctg{\pi}(s') \Bigr)
    \\
    \label{eq:bell-time}
    \policytime{\pi}(s)
    & =
    1 + \sum_{a \in A} \sum_{s' \in S} \pi(a \mid s) P(s' \mid s,a) \policytime{\pi}(s').
\end{align}

\paragraph{Learning Formulation.}
The learner's goal is to minimize its total cost.
Its performance is measured by the \emph{regret} -- the difference between the learner's total cost in $K$ episodes and the total expected cost of the best \emph{proper} policy in hindsight:
\begin{align*}
    \regret 
    & = 
    \sum_{k=1}^K \sum_{i=1}^{\Ik} c_k(s^k_i,a^k_i) 
    - 
    \min_{\pi \in \proppolicyset} \sum_{k=1}^K \ctg{\pi}_k (\sinit),
\end{align*}
where $\ctg{\pi}_k$ is the cost-to-go of policy $\pi$ with respect to (w.r.t) cost function $c_k$, and $\proppolicyset$ is the set of proper policies.
If $\Ik$ is infinite for some $k$, we define $\regret = \infty$ forcing the learner to reach the goal in every episode.
We also denote by $\pi^\star = \argmin_{\pi \in \proppolicyset} \sum_{k=1}^K \ctg{\pi}_k (\sinit)$ the best policy in hindsight.

Our analysis makes use of the Bellman equations, that hold under the conditions described before Eq.~\eqref{eq:bell-time}.
To make sure these are met, we assume that  the costs are strictly positive.
\begin{assumption}
    \label{ass:c-min}
    All costs are positive, i.e., there exists $\cmin > 0$ such that $c_k(s,a) \geq \cmin$ for every $k$ and $(s,a) \in S \times A$.
\end{assumption}

We can easily eliminate Assumption~\ref{ass:c-min} by applying a perturbation to the instantaneous costs.
That is, instead of $c_k$ we use the cost function $\tilde c_k(s,a) = \max \{ c_k(s,a) , \epsilon \}$ for some $\epsilon > 0$.
This ensures that the effective minimal cost is $\cmin = \epsilon$, at the price of introducing additional bias.
Choosing $\epsilon = \Theta (K^{-1/4})$ ensures that all our algorithms obtain regret bounds of $\wt O(K^{3/4})$ in the general case.
See details in Appendix K and discussion about $\cmin$ in Section~\ref{sec:discussion}.

\paragraph{Occupancy Measures.}
Every policy $\pi$ induces an occupancy measure $q^\pi : S \times A \rightarrow [0,\infty]$ such that $q^\pi(s,a)$ is the expected number of times to visit state $s$ and take action $a$ when playing according to $\pi$, i.e.,
\[
    q^\pi(s,a) 
    = 
    \lim_{T \rightarrow \infty} \bbE \Bigl[ \sum_{t=1}^T \indevent{s_t = s,a_t = a} \mid P,\pi, s_1 = \sinit \Bigr],
\]
where $\indevent{\cdot}$ is the indicator function.
Note that for a proper policy $\pi$, $q^\pi(s,a)$ is finite for every $(s,a)$.
In fact, the correspondence between proper policies and finite occupancy measures is 1-to-1, and its inverse\footnote{If $q(s) = 0$ for some state $s$ then the inverse mapping is not well-defined. 
However, since $s$ will not be reached, we can pick the action there arbitrarily. 
More precisely, the correspondence holds when restricting to reachable states.} for $q$ is given by
$
    \pi^q(a \mid s)
    =
    \frac{q(s,a)}{q(s)}
$ where $q(s) = \sum_{a \in A} q(s,a)$ is the expected number of visits to $s$.
The equivalence between policies and occupancy measures is well-known for MDPs (see, e.g., \cite{zimin2013online}), but also holds for SSPs by linear programming formulation \cite{manne1960linear}.
Notice that the expected cost of policy $\pi$ is linear w.r.t $q^\pi$, i.e.,
\begin{align*}
    \ctg{\pi_k}_k(\sinit)
    & =
    \bbE \Bigl[ \sum_{i=1}^{\Ik} c_k(s^k_i,a^k_i) \mid P,\pi_k,s_1 = \sinit \Bigr]
    \\
    & =
    \sum_{s \in S} \sum_{a \in A} q^{\pi_k}(s,a) c_k(s,a)
    \eqdef
    \langle q^{\pi_k} , c_k \rangle.
\end{align*}
Thus, minimizing the expected regret can be written as an instance of online linear optimization in the following manner,
\begin{align*}
    \bbE [ \regret ]
    & =
    \bbE \Bigl[ \sum_{k=1}^K \ctg{\pi_k}_k(\sinit) - \sum_{k=1}^K \ctg{\pi^\star}_k(\sinit) \Bigr]
    \\
    & =
    \bbE \Bigl[ \sum_{k=1}^K \langle q^{\pi_k} - q^{\pi^\star} , c_k \rangle \Bigr].
\end{align*}

\section{Known Transition Function}
\label{sec:known-P}

We start with the simpler (yet surprisingly challenging) case where $P$ is known to the learner.
Recall that while the transition function is known, the costs change arbitrarily between episodes.
In Section~\ref{sec:omd-ssp} we establish the implementation of the OMD method in SSP, and in Section~\ref{sec:high-prob-known-P} we use it to obtain a  high probability regret bound.

\subsection{Online Mirror Descent for SSP}
\label{sec:omd-ssp}

Online mirror descent is a popular framework for OCO and its application to occupancy measures yields the O-REPS algorithms \cite{zimin2013online,rosenberg2019full,rosenberg2019bandit,jin2019learning}.
Usually these algorithms operate w.r.t to the set of all occupancy measures (which corresponds to the set of all policies), but a naive application of this kind fails in SSP because it does not guarantee that the learner plays proper policies.
For example, in the first episode these algorithms play the uniform policy which may suffer exponential cost (see Appendix A).

Thus, we propose to apply OMD to the set $\boundedocset{\tau}$ -- occupancy measures of policies $\pi$ that reach the goal in expected time $\policytime{\pi}(\sinit) \le \tau$.
This set is convex and has a compact representation as we show shortly.
Our algorithm SSP-O-REPS operates as follows. 
In the beginning of episode $k$, it picks an occupancy measure $q_k$ from $\boundedocset{\tau}$ which minimizes a trade-off between the current cost function and the distance to the previously chosen occupancy measure.
Then, it extracts the policy $\pi_k = \pi^{q_k}$ and plays it through the episode.
Formally,
\begin{align}
    \label{eq:omd-update}
    q_k
    =
    q^{\pi_k}
    = 
    \argmin_{q \in \boundedocset{\tau}} \eta \langle q , c_{k-1} \rangle + \KL{q}{q_{k-1}},
\end{align}
where $\text{KL}(\cdot||\cdot)$ is the KL-divergence, and $\eta > 0$ is a learning rate.
Computing $q_k$ is implemented in two steps: first find the unconstrained minimizer and then project it into $\boundedocset{\tau}$, i.e.,
\begin{align}
    \label{eq:omd-step-1}
    q'_k 
    & = 
    \argmin_q \eta \langle q , c_{k-1} \rangle + \KL{q}{q_{k-1}}
    \\
    \label{eq:omd-step-2}
    q_k
    & = 
    \argmin_{q \in \boundedocset{\tau}} \KL{q}{q'_k}.
\end{align}
Eq.~\eqref{eq:omd-step-1} has a closed form $q'_k(s,a) = q_{k-1}(s,a) e^{- \eta c_{k-1}(s,a)}$, and Eq.~\eqref{eq:omd-step-2} can be formalized as a constrained convex optimization problem with the following linear constraints:
\begin{align}
    \nonumber
    \forall s. \sum_{a \in A} q(s,a) 
    -
    \sum_{s' \in S} \sum_{a' \in A} q(s',a') P(s | s',a')
    & =
    \indevent{s = \sinit}
    \\
    \label{eq:const-time}
    \sum_{s \in S} \sum_{a \in A} q(s,a) 
    & \le 
    \tau,
\end{align}
where we omitted non-negativity constraints.
The first set of constraints are standard flow constraints, while the novel constraint~\eqref{eq:const-time} ensures that $\policytime{\pi^q}(\sinit) \le \tau$.
In Appendix B we show how to solve this problem efficiently and describe implementation details for the algorithm. 
Pseudocode in Appendix C.

Finally, we need to pick the parameter $\tau$.
While it needs to upper bound $\policytime{\pi^\star}(\sinit)$ in order to have $q^{\pi^\star} \in \boundedocset{\tau}$, we want it to be as small as possible to get tighter regret guarantees.
To that end, define the SSP-diameter \cite{tarbouriech2019noregret} $D = \max_{s \in S} \min_{\pi \in \proppolicyset} \policytime{\pi}(s)$ and pick $\tau = D/\cmin$.
The diameter can be computed efficiently by finding the optimal policy w.r.t the constant cost function $c(s,a) = 1$ (see Appendix B).
We refer to this policy as the fast policy $\fastpolicy$, and it holds that $D = \max_{s \in S} \policytime{\fastpolicy}(s)$.

Indeed $q^{\pi^\star} \in \boundedocset{D/\cmin}$ because the total cost of the best policy in hindsight in $K$ episodes is upper bounded by the total cost of any other policy, e.g., the fast policy (which is at most $DK$), and is lower bounded by the expected time of $\pi^\star$ times the minimal cost, i.e., $\ctg{\pi^\star}_k (\sinit) \ge \cmin \policytime{\pi^\star}(\sinit)$ (see Appendix D).
In Appendix A we also show that this choice of $\tau$ cannot be smaller in general.

In Appendix D we provide the full analysis of the algorithm yielding the following regret bound in expectation.
Moreover, we show that all the chosen policies must be proper and therefore the goal is reached with probability $1$ in all episodes.

\begin{theorem}
    \label{thm:exp-reg-full-info}
    Under Assumption~\ref{ass:c-min}, the expected regret of SSP-O-REPS with known transition function and $\eta = \wt \Theta (\frac{1}{\sqrt{K}})$ is
    \[
        \bbE [ \regret ]
        \le
        O \Bigl( \frac{D}{\cmin} \sqrt{K \log \frac{D |S| |A|}{\cmin}} \Bigr)
        =
        \tO \Bigl( \frac{D}{\cmin} \sqrt{K} \Bigr).
    \]
\end{theorem}

\subsection{High Probability Regret Bound}
\label{sec:high-prob-known-P}

To obtain high probability regret bounds, we must control the deviation between the learner's suffered cost and its expected value.
While this is easily achievable in the finite-horizon setting through an application of Azuma inequality, it appears a major challenge in SSP since there is no finite upper bound on the learner's cost.
In fact, Appendix A illustrates a simple example with $0$ expected regret, but constant probability to suffer large regret (linear in $K$).
The idea here is that even though a policy has small cost in expectation, there might be a tiny probability that it suffers huge cost (this cannot happen in finite-horizon since the cost is always bounded by $H$).
Finally, even an event with tiny probability will happen at least once if there is a large number of episodes $K$.

Our strategy to control the deviation between the learner's actual suffered cost and its expected value is based on the observation that this quantity is closely related to the expected time to reach the goal from any state.
This is illustrated by the following lemma whose proof is based on an adaptation of Azuma inequality to unbounded martingales (Theorem 11) which may be of independent interest.

\begin{lemma}
    \label{lem:dev-from-exp-cost}
    Assume that in each episode $k$ the learner plays a strategy $\sigma_k$ such that the expected time to reach the goal from any state is at most $\tau$.
    Then, with probability at least $1 - \delta$,
    \begin{align*}
        \sum_{k=1}^K \sum_{i=1}^{\Ik} c_k(s^k_i,a^k_i)
        & \le
        \sum_{k=1}^K \bbE \Bigl[ \sum_{i=1}^{\Ik} c_k(s^k_i,a^k_i) \mid P,\sigma_k,s_1^k = \sinit \Bigr]
        \\
        & \qquad +
        O \Bigl( \tau \sqrt{K \log^3 \frac{K}{\delta}} \Bigr).
    \end{align*}
\end{lemma}

Thus, bounding the regret in high probability boils down to guaranteeing that $\policytime{\pi_k}(s) \le D/\cmin$ for all $s \in S$ and not just $\sinit$.
Unfortunately, these constraints admit a non-convex set of occupancy measures.
To bypass this issue we propose the SSP-O-REPS2 algorithm that operates as follows:
start every episode $k$ by playing the policy $\pi_k$ chosen by SSP-O-REPS (i.e., Eq.~\eqref{eq:omd-update}), but once we reach a state $s$ whose expected time to the goal is too long (i.e., $\policytime{\pi_k}(s) \ge D/\cmin$), switch to the fast policy $\fastpolicy$.
We defer to the pseudocode in Appendix E.

Now the conditions of Lemma~\ref{lem:dev-from-exp-cost} are clearly met, so it remains to relate the expected cost of our new strategy $\sigma_k$ to this of $\pi_k$.
The key novelty of our mid-episode policy switch is the timing.
The naive approach would be to perform the switch when the policy takes too long, but then there is no way to bound the excess cost when compared to that of $\pi_k$. 
Performing the switch only once a ``bad'' state is reached ensures that the expected cost of $\sigma_k$ can only be better than $\pi_k$.
The analysis in Appendix F makes these claims formal and proves the following high probability regret bound.

\begin{theorem}
    \label{thm:hp-reg-full-info}
    Under Assumption~\ref{ass:c-min}, with probability $1 - \delta$, the regret of SSP-O-REPS2 with known transition function is
    \[
        \regret
        \le
        O \Bigl( \frac{D}{\cmin} \sqrt{K \log^3 \frac{K D |S| |A|}{\delta \cmin}} \Bigr)
        =
        \tO \Bigl( \frac{D}{\cmin} \sqrt{K} \Bigr).
    \]
\end{theorem}

\section{Unknown Transition Function}
\label{sec:high-prob-unknown-P}

A standard technique to deal with unknown transition function in adversarial MDPs is to use optimistic estimates of $P$.
We follow this approach but, as in the known transitions case, crucial modifications are necessary to apply optimism and obtain regret guarantees.
In this section we describe our SSP-O-REPS3 algorithm for unknown transitions.

We start by describing the confidence sets and transition estimates used by the algorithm.
SSP-O-REPS3 proceeds in \emph{epochs} and updates the confidence set at the beginning of every epoch.
The first epoch begins at the first time step, and an epoch ends once an episode ends or the number of visits to some state-action pair is doubled.
Denote by $N^e(s,a)$ the number of visits to $(s,a)$ up to (and not including) epoch $e$, and by $N^e(s,a,s')$ the number of times this was followed by a transition to $s'$.
Let $N^e_+(s,a) = \max \{ N^e(s,a) , 1 \}$ and define the empirical transition function for epoch $e$ by $\bar P_e(s' | s,a) =N^e(s,a,s')/N^e_+(s,a)$.
Finally, define the confidence set for epoch $e$ as the set of all transition functions $P'$ such that for every $(s,a,s') \in S \times A \times (S \cup \{g\})$,
\[
    | P' (s' \mid s,a) - \bar P_e (s' \mid s,a) |
    \le
    \epsilon_e(s' \mid s,a),
\]
where $\epsilon_e(s' | s,a) = 4 \sqrt{\bar P_e (s' | s,a) A^e(s,a)} + 28 A^e(s,a)$ is the confidence set radius for $A^e(s,a) = \frac{\log \bigl(|S| |A| N^e_+(s,a) / \delta \bigr)}{N^e_+(s,a)}$.
By Bernstein inequality (see, e.g., \cite{azar2017minimax}), these confidence sets contain $P$ with probability $1 - \delta$ for all epochs.

Next, we extend our OMD implementation to the unknown transitions case.
We follow the elegant approach of \cite{rosenberg2019full} that use occupancy measures that are extended to include a transition function as well, that is,
\[
    q^{P,\pi}(s,a,s') 
    = 
    \lim_{T \rightarrow \infty} \bbE \Bigl[ \sum_{t=1}^T \indevent{s_t = s,a_t = a,s_{t+1}=s'} \Bigr],
\]
where $\bbE[\cdot]$ is shorthand for $\bbE [\cdot \mid P,\pi,s_1 = \sinit]$ here.
Now an occupancy measure $q$ corresponds to a transition function-policy pair with the inverse mapping given by
\[
    \pi^q(a \mid s) 
    = 
    \frac{q(s,a)}{q(s)}
    \quad ; \quad
    P^q(s' \mid s,a) = \frac{q(s,a,s')}{q(s,a)},
\]
where $q(s,a) = \sum_{s' \in S \cup \{ \ssink \}} q(s,a,s')$ is the expected number of visits to $(s,a)$ w.r.t $P^q$ when playing $\pi^q$.
We extend the set $\boundedocset{\tau}$ (which we cannot compute without knowing $P$), and perform OMD on the set $\wtboundedocset{\tau}{e}$ that changes through epochs.
$\wtboundedocset{\tau}{e}$ is defined as the set of occupancy measures $q$ whose induced transition function $P^q$ is in the confidence set of epoch $e$ and the expected time of $\pi^q$ (w.r.t $P^q$) from $\sinit$ to the goal is at most $\tau$.
This set is again convex with a compact representation, and it admits the following OMD update step,
\begin{align}
    \label{eq:comp-q-unknown}
    q_k
    =
    q^{P_k,\pi_k}
    = 
    \argmin_{q \in \wtboundedocset{\tau}{e(k)}} \eta \langle q , c_{k-1} \rangle + \KL{q}{q_{k-1}},
\end{align}
where $e(k)$ denotes the first epoch in episode $k$.
Similarly to the known transitions case, this update can be performed efficiently.
See Appendix G for details of the implementation.

In contrast to the known transitions case, this version of OMD cannot even guarantee bounded regret in expectation, because without knowledge of the transition function there is no guarantee that the chosen policies are even proper.
Note that in the easier loop-free SSP setting, this OMD version is enough to guarantee a high probability regret bound even with unknown transitions.
We now describe the mechanisms that need to be combined with OMD to obtain our regret bound.

Similarly to Section~\ref{sec:high-prob-known-P}, we must make sure that the learner does not take too much time to reach the goal.
The problem now is that we cannot compute its expected time $\policytime{\pi_k}$ since $P$ is unknown.
Instead, we use the expected time of $\pi_k$ w.r.t $P_k$ (denoted by $\optimisticpolicytime{\pi_k}_k$) which is an estimate of $\policytime{\pi_k}$, but not necessarily an optimistic one.
Once a state $s$ is reached such that $\optimisticpolicytime{\pi_k}_k(s) \ge D/\cmin$ we want to switch to the fast policy $\fastpolicy$ which again cannot be computed without knowing $P$.
This policy is replaced with its optimistic estimate $\optimisticfastpolicy{e}$, which we refer to as the optimistic fast policy.
Together with the optimistic fast transition function $\wt P^f_e$, this policy minimizes the expected time to the goal out of all pairs of policies and transition functions from the confidence set of epoch $e$.
The details of computing the optimistic fast policy are in Appendix G.

If we were in the known transitions case, this would have been enough.
So it seems that it should also suffice with unknown transitions, if we recompute the optimistic fast policy in the end of every epoch similarly to \cite{cohen2020ssp}.
However, in the adversarial setting this approach fails for two main reasons.
First, we cannot guarantee that $\optimisticpolicytime{\pi_k}_k$ is a good enough estimate of $\policytime{\pi_k}$ in all states.
Second, the learner's policy is stochastic which means that we cannot guarantee all actions are being explored enough (as opposed to \cite{cohen2020ssp} that only play deterministic policies since they do not tackle adversarial costs).
To overcome these challenges, we propose to force exploration in the following manner.
Define a state to be \emph{unknown} until every action was played at least $\Phi = \alpha \numvisitsuntilknownbern$ times in this state (for some constant $\alpha > 0$), and \emph{known} afterwards.
When reaching an unknown state, we play the least played action so far (forcing exploration), and only then switch to the optimistic fast policy.
The idea behind this forced exploration is inspired by \cite{cohen2020ssp} that show that once all states are known, the optimistic fast policy is proper with high probability.

To summarize, SSP-O-REPS3 operates as follows.
We start each episode $k$ by playing the policy $\pi_k$ computed in Eq.~\eqref{eq:comp-q-unknown}, and maintain confidence sets that are updated at the beginning of every epoch.
When we reach a state $s$ such that $\optimisticpolicytime{\pi_k}_k(s) \ge D/\cmin$, we switch to the optimistic fast policy.
In addition, when an unknown state is reached we play the least played action up to this point and then switch to the optimistic fast policy.
Finally, we also make the switch to the optimistic fast policy once the number of visits to some state-action pair is doubled, at which point we also recompute it.
We defer to the full pseudocode in Appendix H and to the full analysis in Appendix I that yields the following regret bound.

\begin{theorem}
    \label{thm:reg-bound-unknown-P}
    Under Assumption~\ref{ass:c-min}, with probability $1 - \delta$, the regret of SSP-O-REPS3 with known SSP-diameter $D$ is
    \begin{align*}
        \regret
        & \le
        \tO \Bigl( \frac{D |S|}{\cmin} \sqrt{|A| K} + \frac{D^2 |S|^2 |A|}{\cmin^2} \Bigr)
        \\
        & =
        \tO \Bigl( \frac{D |S|}{\cmin} \sqrt{|A| K} \Bigr),
    \end{align*}
    where the last equality holds for $K \ge D^2 |S|^2 |A| / \cmin^2$.
\end{theorem}

Our analysis builds on ideas from \cite{cohen2020ssp} that analyze optimistic algorithms in SSP with stochastic costs.
However, for the many reasons described in this paper and because our algorithm is not optimistic, many novel technical adaptions are needed in order to tackle the new challenges that arise when both the costs are adversarial and the transition function is unknown.
Due to lack of space these are mostly presented in Appendix I, but here we give a short overview of the analysis.

Recall that the learner has two objectives in SSP: minimizing cost and reaching the goal.
When transitions were known, we used Lemma~\ref{lem:dev-from-exp-cost} to say that (with high probability) the goal is reached in every episode, and then we could simply focus on bounding the regret.
With unknown transitions, the argument for bounding the total time becomes more involved.
The idea is that (with high probability) the number of steps between policy switches cannot be too long, as a consequence of our added mechanisms. 
To that end, we split the time steps into \emph{intervals}.
The first interval begins at the first time step, and an interval ends once (1) an episode ends, (2) an epoch ends, (3) an unknown state is reached, or (4) a policy switch is made due to reaching a ``bad'' state.
Intuitively, we bound the length of every interval by $\tO(D/\cmin)$ with high probability, and then use fact that the number of intervals is bounded by $\tO(K + D|S|^2|A|/\cmin^2)$ to bound the total time.
Then, we show that the regret of the learner can be bounded by the regret of OMD (analyzed in Section~\ref{sec:known-P}) plus the square root of the total variance (times $|S|^2 |A|$).
Finally, we obtain our regret bound by noticing that the total variance is equal to the variance in each interval times the number of intervals, and bounding the variance in an interval by $O(D^2/\cmin^2)$ .




\paragraph{Estimating the SSP-diameter.}
When the transition function is unknown, we cannot compute the diameter $D$.
However, a careful look at our algorithms shows that we use it only twice.
First, we pick $\tau = D/\cmin$ as an upper bound on the expected time of the best policy in hindsight.
For this purpose it is enough to use $\policytime{\fastpolicy}(\sinit)/\cmin$, and therefore we shall dedicate the first $L$ episodes to computing an estimate $\wt D(\sinit)$ of $\policytime{\fastpolicy}(\sinit)$ before running SSP-O-REPS3.
Second, $D$ is used to make a switch when a ``bad'' or unknown state $s$ is reached, but again it is enough to use $\policytime{\fastpolicy}(s)$ instead.
Similarly, we use the first $L$ visits to $s$ to estimate $\policytime{\fastpolicy}(s)$ and then continue executing the algorithm with $\wt D(s)$ instead of $D$.

To compute $\wt D(s)$ we run the algorithm of \cite{cohen2020ssp} for regret minimization in SSP with constant cost of $1$ (since it measures time).
By their regret bound, we can set $L \approx \sqrt{K}$ and suffer negligible additional regret.
This is also enough to yield the two properties we need in order to keep the same regret bound (with high probability): $\wt D(s)$ is an upper bound on $\policytime{\fastpolicy}(s)$ for any $s \in S$, and $\wt D(s) \le O(D)$ (i.e., it is not too large).
Details and full proofs in Appendix~J.

\section{Discussion}
\label{sec:discussion}

\paragraph{Lower bound and future work.}
In this paper we presented the first algorithms to achieve sub-linear regret in SSP with adversarially changing costs.
Building on some of our ideas, \cite{chen2020minimax} recently proposed sophisticated algorithms with minimax optimal regret of $\tO (\sqrt{D T_\star K})$ in the known transitions case, where $T_\star$ is the expected time of the best policy in hindsight.
Interestingly, their lower bound reveals a gap from the stochastic setting (and from finite-horizon adversarial MDPs), showing that the adversarial SSP model is indeed significantly more challenging than previous models.
Moreover, it shows that our regret bounds are near-optimal (up to $1/\sqrt{\cmin}$) in the hard case where the expected time of $\pi^\star$ is as large as $D/\cmin$ (see example in Appendix~A).

There are still many interesting open problems in adversarial SSPs.
Achieving minimax optimal regret with unknown dynamics is an important open problem that can hopefully be solved using some of the techniques presented here.
The known transitions case is still far from solved as well.
The algorithm of \cite{chen2020minimax} requires knowing $T_\star$ in advance which is a very restrictive assumption.
Estimating $T_\star$ on the fly is another important open problem which seems very challenging due to the adversarially changing costs.

\paragraph{SSP vs finite-horizon.}
As this paper and the works of \cite{tarbouriech2019noregret,cohen2020ssp} attempt to show, the SSP problem presents very different challenges than finite-horizon MDPs (or equivalently loop-free SSPs) although they are seemingly similar in structure.
These differences stem from the double objective that the agent has to face in SSP, i.e., minimizing cost vs reaching the goal, while the only focus of the finite-horizon model is minimizing cost (the time of each episode is bounded by $H$ by definition).
Apart from the conceptual difference, this leads to numerous technical challenges, where the biggest one is unbounded value functions and episode lengths.
Note that almost every online learning problem has some boundness assumptions and therefore novel technical tools must be used here (or at least non-trivial adaptations of existing tools, e.g., Theorem 11).

Dealing with adversarial costs in SSP is challenging even when the transition function is known to the learner.
As described in this paper, using occupancy measures, this becomes an online linear optimization problem.
However, unlike the finite-horizon case, in the SSP setting the decision set (i.e., the set of occupancy measures) does not have a bounded diameter (in finite-horizon it has diameter $H$), and this is the source of the unique challenges.
To address these issues, we proposed to limit the decision set so it has a finite diameter (but still contains the best occupancy measure in hindsight).
Surprisingly this is not enough to obtain high probability regret bounds (see example in Appendix A), because we cannot constrain the expected time from all states, and thus we used a novel notion of switching policy when reaching ``bad'' states.

When the transitions are unknown, all these challenges become harder because in order to estimate the expected cost of a policy to reasonable error (even just to determine whether it is proper), one needs very good estimation of the transition function.
While in the finite-horizon setting OMD is easily generalized to unknown transitions through optimistic estimates, in adversarial SSP further adaptations are necessary.

\paragraph{Adversarial vs stochastic costs in SSP.}
In this paper we studied the effects of adversarially changing costs on the general SSP model without any restrictive assumptions, previously studied only under stochastic costs \cite{tarbouriech2019noregret,cohen2020ssp}.
The recent lower bound \cite{chen2020minimax} shows that adversarial costs in SSP pose significant new challenges, as opposed to finite-horizon where the lower bound for adversarial or stochastic costs is the same.

Both \cite{tarbouriech2019noregret,cohen2020ssp} use optimism w.r.t the costs, which ensures them that the time is also bounded since they use the positive costs assumption, i.e., Assumption~\ref{ass:c-min}.
While $\cmin$ appears in their regret bounds, the latter is able to push it to an additive term (independent of $K$) and thus keep a regret of $\tO(\sqrt{K})$ in the general case (after applying perturbation).
Since we are dealing with adversarial costs, we cannot use optimism.
Instead we use the OMD method to handle the adversary, and must make sure that we do so while reaching the goal with high probability.
For this reason we incorporate explicit constraints on the time, and these cause us to suffer regret that depends on $D/\cmin$ instead of $D$ since $\policytime{\pi^\star}$ is not bounded by $D$ even though $\ctg{\pi^\star}$ is.
This dependence is unavoidable in the adversarial case, and it also requires the additional challenge of estimating $D$, while optimistic estimates are bounded by $D$ (with high probability).

Technically, our analysis follows the framework of \cite{cohen2020ssp} since we need to show the goal is reached with high probability.
Yet, the mechanisms we introduced are necessary to make this framework useful in the adversarial case, and even then careful analysis is needed.
Hopefully, the framework we introduced here will help obtain minimax optimal regret with unknown transitions.
In this context, two notable mechanisms are forced exploration and policy switch in ``bad'' states.
Forced exploration is key to handle large variance stochastic policies might have in SSP (without adversarial costs deterministic policies suffice).
It ensures that we can determine whether our policies are proper as soon as possible and finish intervals early.
While the motivation for switching in ``bad'' states is clear from known transitions, when dynamics are unknown this switch becomes problematic as we cannot guarantee it actually occurs in ``bad'' states (our estimate for the time is not even optimistic).
More ideas are required in order to bound the excess cost that comes from switching policies in falsely estimated ``bad'' states (see Appendix I).

\section*{Acknowledgments}

This project has received funding from the European Research Council (ERC) under the European Union’s Horizon 2020 research and innovation program (grant agreement No. 882396), by the Israel Science Foundation (grant number 993/17) and the Yandex Initiative for Machine Learning at Tel Aviv University.

\bibliographystyle{named}
\bibliography{ijcai21}

\clearpage
\onecolumn
\input{supp.tex}

\end{document}

%% file: supp.tex
\appendix

\section{Examples that illustrate some challenges in adversarial SSPs}
\label{sec:bad-examples}

\subsection{Naive application of OMD fails in SSP}

In general, the first policy that OMD picks is the one that maximizes the entropy, which is the uniform policy, i.e., $\pi^u(a \mid s) = 1/|A|$ for every $(s,a) \in S \times A$.
Next we show that, in SSP, this might result in exponential cost of $|A|^{|S|}$ already in the first episode.
In the finite-horizon setting, this is not a concern because the cost in a single episode is always bounded by $H$, while in SSP it can be infinite.

Consider the following MDP $M = (S,A,P,\sinit,\ssink)$ with the state space $S = \{1,\dots,|S|\}$.
In every state $i$ there is one action $a(i)$ (picked uniformly at random in advance) such that $P(i+1 \mid i,a(i)) = 1$, while the other actions return the agent to the initial state $\sinit = 1$, i.e., $P(1 \mid i,a) = 1$ for every $a \ne a(i)$.
Finally, the cost function (for the first episode in which OMD picks $\pi^u$) is simply $c(s,a) = 1$ for every $(s,a) \in S \times A$.

Clearly the best policy in this case is to pick $a(i)$ in state $i$ and then the total cost is $|S|$ (the SSP-diameter in this example is also $|S|$).
However, the uniform policy picks this action only with probability $1/|A|$ which yields exponential expected time to reach the goal (and therefore exponential cost).
To see that consider the Bellman equations for $\pi^u$:
\begin{align*}
    \ctg{\pi^u}(i) 
    & = 
    1 + \frac{1}{|A|} \cdot \ctg{\pi^u}(i+1) + (1 - \frac{1}{|A|}) \cdot \ctg{\pi^u}(1)
    \qquad \forall i=1,\dots,|S|-1
    \\
    \ctg{\pi^u}(|S|)
    & =
    1 + \frac{1}{|A|} \cdot 0 + (1 - \frac{1}{|A|}) \cdot \ctg{\pi^u}(1).
\end{align*}
Solving these equations gives $\ctg{\pi^u}(\sinit) = \ctg{\pi^u}(1) = \frac{|A| (|A|^{|S|} - 1)}{|A| - 1} \ge |A|^{|S|}$.

\subsection{The expected time of the best policy in hindsight might be \texorpdfstring{$\Omega(D/\cmin)$}{}}

The following example shows that the expected time of the best policy in hindsight might be $\Omega(D/\cmin)$, and therefore there is no better apriori choice for $\tau$.

Consider the MDP $M = (\{\sinit\}, \{a_1,a_2 \},P,\sinit,\ssink)$ that has only one state (other than the goal) and two actions.

Playing action $a_1$ transitions to the goal with probability $1/D$ and back to $\sinit$ with probability $1 - 1/D$, i.e., $P(\sinit \mid \sinit,a_1) = 1 - 1/D$ and $P(\ssink \mid \sinit,a_1) = 1/D$.
Therefore, the expected time of the policy that plays $a_1$ is $D$ and so the SSP-diameter is also bounded by $D$.

Playing action $a_2$ transitions to the goal with probability $2 \cmin / D$ and back to $\sinit$ with probability $1 - 2 \cmin / D$, i.e., $P(\sinit \mid \sinit,a_2) = 1 - 2 \cmin / D$ and $P(\ssink \mid \sinit,a_2) = 2 \cmin / D$.
Therefore, the expected time of the policy that plays $a_1$ is $\nicefrac{D}{2 \cmin}$.

Apriori there is no way to tell if $a_1$ or $a_2$ will be the best policy in hindsight.
For example, if $c(\sinit,a_1) = 1$ and $c(\sinit,a_2) = \cmin$ then $a_2$ is better, and if $c(\sinit,a_1) = 1$ and $c(\sinit,a_2) = 3 \cmin$ then $a_1$ is better.
Thus, the smallest possible choice for $\tau$ in this case is $\nicefrac{D}{2 \cmin} = \Omega(D/\cmin)$.

\subsection{A bound on the expected regret does not guarantee a high probability regret bound in SSP}

In most online learning problems, algorithms that guarantee bounded regret in expectation also guarantee bounded regret with high probability.
The way to show this (in most problems) is by Azuma inequality for bounded martingales.
However, the SSP problem is unique in the sense that guaranteeing bounded regret in expectation is significantly easier than guaranteeing bounded regret with high probability.
This is illustrated by the following simple example in which there exists a policy with $0$ expected regret, but linear regret with constant probability of at least $\nicefrac{1}{30}$.

Consider the MDP $M = (\{\sinit,s_1\},\{a_1,a_2\},P,\sinit,\ssink)$ that has only two states (other than the goal) and two actions.
In state $\sinit$ playing action $a_1$ simply transitions to the goal, i.e., $P(\ssink \mid \sinit,a_1) = 1$.
In this state playing action $a_2$ transitions to the goal with probability $p = 1 - \frac{1 - \cmin}{10 K}$ and transitions to state $s_1$ with probability $1 - p$, i.e., $P(\ssink \mid \sinit,a_2) = p$ and $P(s_1 \mid \sinit,a_2) = 1-p$.
Moreover, in state $s_1$ both actions have the same effect. 
They transition to the goal with probability $\nicefrac{1}{10 K}$ and remain in state $s_1$ with probability $1- \nicefrac{1}{10 K}$, i.e., $P(\ssink \mid s_1,a_i) = \nicefrac{1}{10 K}$ and $P(s_1 \mid s_1,a_i) = 1 - \nicefrac{1}{10 K}$ for $i=1,2$.

Now consider the simple case where the cost function is the same for all episodes.
Playing action $a_1$ always suffers a cost of $1$, i.e., $c(\sinit,a_1) = c(s_1,a_1) = 1$.
Playing action $a_2$ suffers cost of $\cmin$ in $\sinit$ but cost of $1$ in $s_1$, i.e., $c(\sinit,a_2) = \cmin$ and $c(s_1,a_2) = 1$.
There are only two policies: $\pi_1$ plays action $a_1$ in state $\sinit$, and $\pi_2$ plays $a_2$.
Notice that both policies have the same expected cost since clearly $\ctg{\pi_1}(\sinit) = 1$ and
\[
    \ctg{\pi_2}(\sinit)
    =
    \cmin + p \cdot 0 + (1-p) \cdot 10K
    =
    \cmin + \frac{1 - \cmin}{10 K} \cdot 10K
    =
    1.
\]
Moreover, both have similar expected time since clearly $\policytime{\pi_1}(\sinit) = 1$ and 
\[
    \policytime{\pi_2}(\sinit)
    =
    1 + p \cdot 0 + (1-p) \cdot 10 K
    =
    1 + \frac{1 - \cmin}{10 K} \cdot 10 K
    =
    2 - \cmin
    \le
    2.
\]

Thus, playing policy $\pi_2$ in all episodes has optimal expected regret of $0$ since
\[
    \bbE [\regret] 
    = 
    \bbE \Bigl[ \sum_{k=1}^K \ctg{\pi_2}(\sinit) - 1 \Bigr] 
    =
    \bbE \Bigl[ \sum_{k=1}^K 1 - 1 \Bigr]
    =
    0.
\]

However, we now show that with probability at least $1/2$ the actual regret is linear.
Define the event $E_k$ -- in episode $k$ the agent's cost was at most $2K$.
Now define $E = \bigcap_{k=1}^K E_k$ as the event that $E_k$ occurs for all episodes.
Notice that if $E$ does not occur than the regret is linear in $K$ since in some episode $k$ the cost was at least $2K$ while the overall cost of $\pi_1$ in all episodes is just $K$.
The following lemma proves that event $E_k$ occurs with probability at most $1 - \nicefrac{1}{26 K}$ and therefore event $E$ indeed occurs with probability at most $(1 - \nicefrac{1}{26 K})^K \le e^{-1/26} \le \nicefrac{29}{30}$.

\begin{lemma}
    For every $k=1,\dots,K$ it holds that $\Pr [E_k] \le 1 - \nicefrac{1}{26 K}$.
\end{lemma}

\begin{proof}
    Recall that $E_k$ is the event that the actual cost of the learner in episode $k$ is bounded by $2K$.
    The probability of that is the probability to transition to the goal from $\sinit$ or to transition to $s_1$ and stay there for at most $2K$ steps.
    Thus,
    \begin{align*}
        \Pr [E_k]
        & \le
        p + (1-p) \sum_{i=1}^{2K} (1 - \frac{1}{10K})^i \cdot \frac{1}{10 K}
        \\
        & =
        1 - \frac{1 - \cmin}{10K} + \frac{1 - \cmin}{100K^2} \sum_{i=1}^{2K} (1 - \frac{1}{10K})^i
        \\
        & \le
        1 - \frac{1 - \cmin}{10K} + \frac{1 - \cmin}{100K^2} \cdot \frac{1 - (1 - \frac{1}{10K})^{2K+1}}{\nicefrac{1}{10K}}
        \\
        & =
        1 - \frac{1 - \cmin}{10K} + \frac{1 - \cmin}{10K} \cdot \bigl( 1 - (1 - \frac{1}{10K})^{2K+1} \bigr)
        \\
        & \le
        1 - \frac{1 - \cmin}{10 K} + \frac{1 - \cmin}{10 K} \cdot \frac{1}{5}
        \le
        1 - \frac{1 - \cmin}{13 K}
        \le
        1 - \frac{1}{26 K},
    \end{align*}
    where the third inequality holds for large enough $K$ since $(1 - \frac{1}{10K})^{2K+1} \rightarrow e^{-1/5}$, and the last inequality holds for $\cmin \le 1/2$.
\end{proof}

\newpage

\section{Implementation details for SSP-O-REPS}
\label{sec:imp-det-1}

\subsection{Computing \texorpdfstring{$q_k$}{}}

Before describing the algorithm, some more definitions are in order.
First, define $\KL{q}{q'}$ as the unnormalized Kullback–Leibler divergence between two occupancy measures $q$ and $q'$:
\[
    \KL{q}{q'} 
    = 
    \sum_{s \in S} \sum_{a \in A} q(s,a) \log \frac{q(s,a)}{q'(s,a)} + q'(s,a) - q(s,a).
\]
Furthermore, let $R(q)$ define the unnormalized negative entropy of the occupancy measure $q$:
\[
    R(q) 
    = 
    \sum_{s \in S} \sum_{a \in A} q(s,a) \log q(s,a) - q(s,a). 
\]
SSP-O-REPS chooses its occupancy measures as follows:
\begin{align*}
    q_1
    & =
    q^{\pi_1} 
    = 
    \argmin_{q \in \boundedocset{D/\cmin}} R(q)
    \\
    q_{k+1}
    & =
    q^{\pi_{k+1}}
    = 
    \argmin_{q \in \boundedocset{D/\cmin}} \eta \langle q , c_k \rangle + \KL{q}{q_k}.
\end{align*}

As shown by \cite{zimin2013online}, each of these steps can be split into an unconstrained minimization step, and a projection step.
Thus, $q_1$ can be computed as follows:
\begin{align*}
    q'_1
    & =
    \argmin_{q} R(q)
    \\
    q_1
    & =
    \argmin_{q \in \boundedocset{D/\cmin}} \KL{q}{q'_1},
\end{align*}
where $q'_1$ has a closed-from solution $q'_1(s,a) = 1$ for every $s \in S$ and $a \in A$.
Similarly, $q_{k+1}$ is computed as follows for every $k=1,\dots,K-1$:
\begin{align*}
    q'_{k+1}
    & =
    \argmin_{q} \eta \langle q , c_k \rangle + \KL{q}{q_k}
    \\
    q_{k+1}
    & =
    \argmin_{q \in \boundedocset{D/\cmin}} \KL{q}{q'_{k+1}},
\end{align*}
where again $q'_{k+1}$ has a closed-from solution $q'_{k+1}(s,a) = q_k(s,a) e^{- \eta c_k(s,a)}$ for every $s \in S$ and $a \in A$.

Therefore, we just need to show that the projection step can be computed efficiently (the implementation follows \cite{zimin2013online}).
We start by formulating the projection step as a constrained convex optimization problem:
\begin{align*}
    \min_q
    &
    \quad \KL{q}{q'_{k+1}}
    &
    \\
    s.t.
    &
    \quad \sum_{a \in A} q(s,a) - \sum_{s' \in S} \sum_{a' \in A} P(s \mid s',a') q(s',a') = \indevent{s = \sinit}
    &
    \forall s \in S
    \\
    &
    \quad \sum_{s \in S} \sum_{a \in A} q(s,a) \le \frac{D}{\cmin}
    &
    \\
    &
    \quad q(s,a) \ge 0
    &
    \forall (s,a) \in S \times A
\end{align*}

To solve the problem, consider the Lagrangian:
\begin{align*}
    \mathcal{L}(q,\lambda,v)
    & =
    \KL{q}{q'_{k+1}} + \lambda \left( \sum_{s \in S} \sum_{a \in A} q(s,a) - \frac{D}{\cmin} \right)
    \\
    & \qquad + 
    \sum_{s \in S} v(s) \left( \sum_{s' \in S} \sum_{a' \in A} P(s \mid s',a') q(s',a') + \indevent{s = \sinit} - \sum_{a \in A} q(s,a) \right)
    \\
    & =
    \KL{q}{q'_{k+1}} + \sum_{s \in S} \sum_{a \in A} q(s,a) \left( \lambda + \sum_{s' \in S} P(s' \mid s,a) v(s') - v(s) \right)
    \\
    & \qquad +
    v(\sinit) - \lambda \frac{D}{\cmin}
\end{align*}
where $\lambda$ and $\{ v(s) \}_{s \in S}$ are Lagrange multipliers.
Differentiating the Lagrangian with respect to any $q(s, a)$, we get
\[
    \frac{\partial \mathcal{L} (q,\lambda,v)}{\partial q(s,a)}
    =
    \log q(s,a) - \log q'_{k+1}(s,a) + \lambda + \sum_{s' \in S} P(s' \mid s,a) v(s') - v(s).
\]

Hence, setting the gradient to zero, we obtain the formula for $q_{k+1}(s, a)$:
\begin{align}
    \nonumber
    q_{k+1}(s,a)
    & =
    q'_{k+1}(s,a) e^{-\lambda - \sum_{s' \in S} P(s' \mid s,a) v(s') + v(s)}
    \\
    \nonumber
    & =
    q_k(s,a) e^{-\lambda -\eta c_k(s,a) - \sum_{s' \in S} P(s' \mid s,a) v(s') + v(s)}
    \\
    \label{eq:q-k-update}
    & =
    q_k(s,a) e^{-\lambda + B_k^v(s,a)},
\end{align}
where the second equality follows from the formula of $q'_{k+1}(s,a)$, and setting $c_0(s,a) = 0$ and $q_0(s,a) = 1$ for every $s \in S$ and $a \in A$.
The last equality follows by defining $B_k^v(s,a) = v(s) - \eta c_k(s,a) - \sum_{s' \in S} P(s' \mid s,a) v(s')$.

We now need to compute the value of $\lambda$ and $v$ at the optimum.
To that end, we write the dual problem $\mathcal{D}(\lambda, v) = \min_q \mathcal{L}(q,\lambda,v)$ by substituting $q_{k+1}$ back into $\mathcal{L}$:
\begin{align*}
    \mathcal{D}(\lambda, v)
    & =
    \sum_{s \in S} \sum_{a \in A} q'_{k+1}(s,a) - \sum_{s \in S} \sum_{a \in A} q_{k+1}(s,a) + v(\sinit) - \lambda \frac{D}{\cmin}
    \\
    & = 
    - \sum_{s \in S} \sum_{a \in A} q_k(s,a) e^{-\lambda +B_k^v(s,a)}
    +
    v(\sinit) - \lambda \frac{D}{\cmin} + \sum_{s \in S} \sum_{a \in A} q'_{k+1}(s,a).
\end{align*}

Now we obtain $\lambda$ and $v$ by maximizing the dual.
Equivalently, we can minimize the negation of the dual (and ignore the term $\sum_{s \in S} \sum_{a \in A} q'_{k+1}(s,a)$), that is:
\begin{align*}
    \lambda_{k+1}, v_{k+1} 
    = 
    \argmin_{\lambda \ge 0,v} \sum_{s \in S} \sum_{a \in A} q_k(s,a) e^{-\lambda +B_k^v(s,a)} + \lambda \frac{D}{\cmin} - v(\sinit).
\end{align*}
This is  a convex optimization problem with only non-negativity constraints (and no constraints about the relations between the variables), which can be solved efficiently using iterative methods like gradient descent.

\subsection{Computing the SSP-diameter and the fast policy}
\label{sec:comp-D}

The fast policy $\fastpolicy$ is a deterministic stationary policy that minimizes the time to the goal state from all states simultaneously (its existence is similar to regular MDPs, for a detailed proof see \cite{bertsekas1991analysis}).
Thus, $\fastpolicy$ is the optimal policy w.r.t the constant cost function $c(s,a) = 1$ for every $s \in S$ and $a \in A$.

Finding the optimal policy of an SSP instance is known as the planning problem.
By \cite{bertsekas1991analysis}, this problem can be solved efficiently using Linear Programming (LP), Value Iteration (VI) or Policy Iteration (PI).

The SSP-diameter $D$ is an upper bound on the expected time it takes to reach the goal from some state, and therefore $D = \max_{s \in S} \policytime{\fastpolicy}(s)$.
Thus, in order to compute $\fastpolicy$ and $D$ we need to perform the following steps:
\begin{enumerate}
    \item Compute the optimal policy $\fastpolicy$ w.r.t the constant cost function $c(s,a)=1$, using LP or VI.
    
    \item Compute $\policytime{\fastpolicy}(s)$ for every $s \in S$ by solving the linear Bellman equations:
    \[
        \policytime{\fastpolicy}(s)
        =
        1 + \sum_{a \in A} \sum_{s' \in S} \fastpolicy(a \mid s) P(s' \mid s,a) \policytime{\fastpolicy}(s')
        \quad \forall s \in S.
    \]
    
    \item Set $D = \max_{s \in S} \policytime{\fastpolicy}(s)$.
\end{enumerate}

\newpage
\section{Pseudo-code for SSP-O-REPS}
\label{sec:code-1}

\begin{algorithm}
    \caption{\sc SSP-O-REPS}
    \label{alg:alg-1}
    \begin{algorithmic} 
        
        \STATE {\bfseries Input:} state space $S$, action space $A$, transition function $P$, minimal cost $\cmin$, optimization parameter $\eta$.
        
         \STATE {\bfseries Initialization:}
         
         \STATE Compute the SSP-diameter $D$ (see \cref{sec:comp-D}).
         \STATE Set $q_0(s,a) = 1$ and $c_0(s,a) = 0$ for every $(s,a) \in S \times A$.

        \smallskip
         
         \FOR{$k=1,2,\ldots$}
         
         \STATE Compute $\lambda_k,v_k$ as follows (using, e.g., gradient descent):
         \[
            \lambda_k, v_k 
            = 
            \argmin_{\lambda \ge 0, v} \sum_{s \in S} \sum_{a \in A} q_{k-1}(s,a) e^{-\lambda +B_{k-1}^v(s,a)} + \lambda \frac{D}{\cmin} - v(\sinit),
         \]
         where $B_k^v(s,a) = v(s) - \eta c_k(s,a) - \sum_{s' \in S} P(s' \mid s,a) v(s')$.
         
         \STATE Compute $q_k$ as follows for every $(s,a) \in S \times A$:
         \[
            q_k(s,a)
            =
           q_{k-1}(s,a) e^{-\lambda_k +B_{k-1}^{v_k}(s,a)}. 
         \]
         
         \STATE Compute $\pi_k$ as follows for every $(s,a) \in S \times A$:
         \[
            \pi_k(a \mid s)
            =
            \frac{q_k(s,a)}{\sum_{a' \in A} q_k(s,a') }.
         \]
         
         \smallskip
         
         \STATE Set $s_1^k \gets \sinit$, $i \gets 1$.
         
        \WHILE{$s_i^k \neq \ssink$}
        
        \STATE Play action according to $\pi_k$, i.e., $a_i^k \sim \pi_k(\cdot \mid s_i^k)$.
        
        \STATE Observe next state $s_{i+1}^k \sim P(\cdot \mid s_i^k,a_i^k)$, $i \gets i+1$.
        \ENDWHILE
        
        \smallskip
        
        \STATE Set $I^k \gets i-1$.
        \STATE Observe cost function $c_k$ and suffer cost $\sum_{j=1}^{I_k}c_k(s_j^k,a_j^k)$.
        \ENDFOR
    \end{algorithmic}
\end{algorithm}

\newpage
\section{Proofs for \texorpdfstring{\cref{sec:omd-ssp}}{}}
\label{sec:omd-ssp-proofs}

\begin{lemma}
    \label{lem:best-in-bounded-oc-set}
    It holds that $q^{\pi^\star} \in \boundedocset{\frac{D}{\cmin}}$.
\end{lemma}

\begin{proof}
    Denote by $\fastpolicy$ the fast policy, i.e., $\fastpolicy = \argmin_{\pi \in \proppolicyset} \policytime{\pi}(\sinit)$.
    By definition of the SSP-diameter we have that $\policytime{\fastpolicy}(\sinit) \le D$.
    Now, recall that $\pi^\star$ is the best policy in hindsight and therefore
    \begin{align}
        \label{eq:opt-better-than-fast}
        \frac{1}{K} \sum_{k=1}^K \ctg{\pi^\star}_k(\sinit)
        \le
        \frac{1}{K} \sum_{k=1}^K \ctg{\fastpolicy}_k(\sinit)
        \le
        \frac{1}{K} \sum_{k=1}^K \policytime{\fastpolicy}(\sinit)
        \le
        D,
    \end{align}
    where the second inequality follows because $c_k(s,a) \le 1$.
    
    However, we also have that $c_k(s,a) \ge \cmin$ and therefore $\ctg{\pi^\star}_k(\sinit) \ge \cmin \policytime{\pi^\star}(\sinit)$.
    Thus, combining with \cref{eq:opt-better-than-fast}, we obtain
    \[
        \cmin \policytime{\pi^\star}(\sinit)
        \le
        \frac{1}{K} \sum_{k=1}^K \ctg{\pi^\star}_k(\sinit)
        \le D.
    \]
    This finishes the proof since $\policytime{\pi^\star}(\sinit) \le \frac{D}{\cmin}$.
\end{proof}

\subsection{Proof of \texorpdfstring{\cref{thm:exp-reg-full-info}}{}}

\begin{lemma}
    \label{lem:bound-R-q-pos}
    Let $\tau \ge 1$.
    For every $q \in \boundedocset{\tau}$ it holds that $R(q) \le \tau \log \tau$.
\end{lemma}

\begin{proof}
    \begin{align*}
        R(q)
        & =
        \sum_{s \in S} \sum_{a \in A} q(s,a) \log q(s,a) - \sum_{s \in S} \sum_{a \in A} q(s,a)
        \\
        & \le
        \sum_{s \in S} \sum_{a \in A} q(s,a) \log q(s,a)
        \\
        & =
        \sum_{s \in S} \sum_{a \in A} q(s,a) \log \frac{q(s,a)}{\tau} + \sum_{s \in S} \sum_{a \in A} q(s,a) \log \tau
        \\
        & \le
        \sum_{s \in S} \sum_{a \in A} q(s,a) \log \tau
        \le
        \tau \log \tau,
    \end{align*}
    where the first two inequalities follow from non-positivity, and the last one from the definition of $\boundedocset{\tau}$.
\end{proof}

\begin{lemma}
    \label{lem:bound-R-q-neg}
    Let $\tau \ge 1$.
    For every $q \in \boundedocset{\tau}$ it holds that $-R(q) \le \tau (1 + \log (|S| |A|))$.
\end{lemma}

\begin{proof}
    Similarly to \cref{lem:bound-R-q-pos} we have that
    \begin{align*}
        - R(q)
        & =
        - \sum_{s \in S} \sum_{a \in A} q(s,a) \log \frac{q(s,a)}{\tau}
        +
        \sum_{s \in S} \sum_{a \in A} q(s,a)
        -
        \sum_{s \in S} \sum_{a \in A} q(s,a) \log \tau
        \\
        & \le
        - \tau \sum_{s \in S} \sum_{a \in A} \frac{q(s,a)}{\tau} \log \frac{q(s,a)}{\tau} + \tau
        \le
        \tau \log (|S| |A|) + \tau,
    \end{align*}
    where the first inequality follows because the last term is non-positive and from the definition of $\boundedocset{\tau}$, and the last inequality follows from properties of Shannon's entropy.
\end{proof}

\begin{proof}[Proof of \cref{thm:exp-reg-full-info}]
    We start with a fundamental inequality of OMD (see, e.g., \cite{zimin2013online}) that holds for every $q \in \boundedocset{D/\cmin}$ (by \cref{lem:best-in-bounded-oc-set} it also holds for $q^{\pi^\star}$),
    \begin{align}
        \label{eq:standard-omd}
        \sum_{k=1}^K \langle q_k - q^{\pi^\star} , c_k \rangle
        \le
        \sum_{k=1}^K \langle q_k - q'_{k+1} , c_k \rangle
        +
        \frac{\KL{q^{\pi^\star}}{q_1}}{\eta}.
    \end{align}
    For the first term we use the exact form of $q'_{k+1}$ and the inequality $e^x \ge 1+x$ to obtain
    \[
        q'_{k+1}(s,a)
        =
        q_k(s,a) e^{- \eta c_k(s,a)}
        \ge
        q_k(s,a) - \eta q_k(s,a) c_k(s,a).
    \]
    We substitute this back and obtain
    \begin{align}
        \nonumber
        \sum_{k=1}^K \langle q_k - q'_{k+1} , c_k \rangle
        & \le
        \eta \sum_{k=1}^K \sum_{s \in S} \sum_{a \in A} q_k(s,a) c_k(s,a)^2
        \le
        \eta \sum_{k=1}^K \sum_{s \in S} \sum_{a \in A} q_k(s,a)
        \\
        \label{eq:omd-term-1}
        & =
        \eta \sum_{k=1}^K \policytime{\pi_k}(\sinit)
        \le
        \eta K \frac{D}{\cmin},
    \end{align}
    where the last inequality follows from the definition of $\boundedocset{D/\cmin}$.
    
    Next we use \cref{lem:bound-R-q-pos,lem:bound-R-q-neg} to bound the second term of \cref{eq:standard-omd}.
    Recall that $q_1$ minimizes $R$ in $\boundedocset{D/\cmin}$, this implies that $\langle \nabla R(q_1), q^{\pi^\star} - q_1 \rangle \ge 0$ because otherwise we could decrease $R$ by taking small step in the direction $q^{\pi^\star} - q_1$.
    Thus we obtain
    \begin{align}
        \nonumber
        \KL{q^{\pi^\star}}{q_1}
        & =
        R(q^{\pi^\star}) - R(q_1) - \langle \nabla R(q_1), q^{\pi^\star} - q_1 \rangle
        \le
        R(q^{\pi^\star}) - R(q_1)
        \\
        \label{eq:omd-term-2}
        & \le 
        \frac{D}{\cmin} \log \frac{D}{\cmin}
        +
        \frac{D}{\cmin} (1 + \log (|S| |A|))
        \le
        \frac{3D}{\cmin} \log \frac{D |S| |A|}{\cmin}.
    \end{align}
    
    By substituting \cref{eq:omd-term-1,eq:omd-term-2} into \cref{eq:standard-omd} and choosing $\eta = \sqrt{\frac{3 \log \frac{D |S| |A|}{\cmin}}{K}}$, we obtain,
    \begin{align}
        \label{eq:reg-olo-bound}
        \sum_{k=1}^K \langle q_k - q^{\pi^\star} , c_k \rangle
        \le
        \eta K \frac{D}{\cmin} + \frac{3D}{\cmin  \eta} \log \frac{D |S| |A|}{\cmin}
        \le
        \frac{2D}{\cmin} \sqrt{3 K \log \frac{D |S| |A|}{\cmin}}.
    \end{align}
    This finishes the proof since
    \[
        \bbE [ \regret ]
        =
        \bbE \left[ \sum_{k=1}^K \langle q_k - q^{\pi^\star} , c_k \rangle \right].
    \]
\end{proof}

\subsection{SSP-O-REPS picks proper policies}

For every policy $\pi_k$ chosen by SSP-O-REPS it holds that $\policytime{\pi_k}(\sinit) \le D/\cmin$.
If there exists some state $s \in S$ such that $\policytime{\pi_k}(s) = \infty$, then the probability to reach it must be zero, since otherwise $\policytime{\pi_k}(\sinit) = \infty$.
Thus there exists $B > 0$ such that if $s$ is reachable from $\sinit$ using $\pi_k$ then $\policytime{\pi_k}(s) \le B$.
By \cref{lem:high-prob-time}, this implies that the goal state will be reached in every episode with probability 1.
Thus, all policies chosen by SSP-O-REPS are proper.

\newpage
\section{Pseudo-code for SSP-O-REPS2}
\label{sec:code-2}

\begin{algorithm}
    \caption{\sc SSP-O-REPS2}
    \label{alg:alg-2}
    \begin{algorithmic} 
        
        \STATE {\bfseries Input:} state space $S$, action space $A$, transition function $P$, minimal cost $\cmin$, optimization parameter $\eta$.
        
         \STATE {\bfseries Initialization:}
         
         \STATE Compute the SSP-diameter $D$ and the fast policy $\fastpolicy$ (see \cref{sec:comp-D}).
         \STATE Set $q_0(s,a) = 1$ and $c_0(s,a) = 0$ for every $(s,a) \in S \times A$.

         \smallskip
         
         \FOR{$k=1,2,\ldots$}
         
         \STATE Compute $\lambda_k,v_k$ as follows (using, e.g., gradient descent):
         \[
            \lambda_k, v_k 
            = 
            \argmin_{\lambda \ge 0, v} \sum_{s \in S} \sum_{a \in A} q_{k-1}(s,a) e^{-\lambda +B_{k-1}^v(s,a)} + \lambda \frac{D}{\cmin} - v(\sinit),
         \]
         where $B_k^v(s,a) = v(s) - \eta c_k(s,a) - \sum_{s' \in S} P(s' \mid s,a) v(s')$.

         \STATE Compute $q_k$ as follows for every $(s,a) \in S \times A$:
         \[
            q_k(s,a)
            =
           q_{k-1}(s,a) e^{-\lambda_k +B_{k-1}^{v_k}(s,a)}. 
         \]

         \STATE Compute $\pi_k$ as follows for every $(s,a) \in S \times A$:
         \[
            \pi_k(a \mid s)
            =
            \frac{q_k(s,a)}{\sum_{a' \in A} q_k(s,a') }.
         \]
         
         \STATE Set $\policytime{\pi_k}(s) \gets \frac{D}{\cmin}$ for every $s \in S$ such that $q^{\pi_k}(s) = \sum_{a \in A} q^{\pi_k}(s,a) = 0$.

         \STATE Compute $\policytime{\pi_k}$ by solving the following linear equations (the Bellman equations):
         \[
            \policytime{\pi_k} (s) 
            = 
            1 + \sum_{a \in A} \sum_{s' \in S} \pi_k(a \mid s) P(s' \mid s,a) \policytime{\pi_k}(s')
            \quad \forall s \in \{ s \in S : \sum_{a \in A} q^{\pi_k}(s,a) >0 \}.
         \]
         
         \smallskip
         
         \STATE Set $s_1^k \gets \sinit$, $i \gets 1$.

        \WHILE{$s_i^k \neq \ssink$ and $\policytime{\pi_k}(s_i^k) < \frac{D}{\cmin}$}
        
        \STATE Play action according to $\pi_k$, i.e., $a_i^k \sim \pi_k(\cdot \mid s_i^k)$.
        
        \STATE Observe next state $s_{i+1}^k \sim P(\cdot \mid s_i^k,a_i^k)$, $i \gets i+1$.

        \ENDWHILE
        
        \smallskip

        \WHILE{$s_i^k \neq \ssink$}
        
        \STATE Play action according to $\fastpolicy$, i.e., $a_i^k \sim \fastpolicy(\cdot \mid s_i^k)$.
        
        \STATE Observe next state $s_{i+1}^k \sim P(\cdot \mid s_i^k,a_i^k)$, $i \gets i+1$.

        \ENDWHILE
        
        \smallskip
        
        \STATE Set $\Ik \gets i-1$.
        
        \STATE Observe cost function $c_k$ and suffer cost $\sum_{j=1}^{I_k}c_k(s_j^k,a_j^k)$.
        \ENDFOR
    \end{algorithmic}
\end{algorithm}

\newpage
\section{Proofs for \texorpdfstring{\cref{sec:high-prob-known-P}}{}}
\label{sec:high-prob-proofs}

\subsection{Proof of \texorpdfstring{\cref{lem:dev-from-exp-cost}}{}}

\begin{lemma}
    \label{lem:high-prob-time}
    Let $\sigma$ be a strategy such that the expected time of reaching the goal state when starting at state $s$ is at most $\tau$ for every $s \in S$.
   Then, the probability that $\sigma$ takes more than $m$ steps to reach the goal state is at most $2e^{-\frac{m}{4\tau}}$.
\end{lemma}

\begin{proof}
    By Markov inequality, the probability that $\sigma$ takes more than $2\tau$ steps before reaching the goal state is at most $1/2$.
    Iterating this argument, we get that the probability that $\sigma$ takes more than $2 k \tau$ steps before reaching the goal state is at most $2^{-k}$ for every integer $k \ge 0$.
    In general, for any $m \ge 0$, the probability that $\sigma$ takes more than $m$ steps before reaching the goal state is at most $2^{-\lfloor\frac{m}{2\tau} \rfloor} \le 2 \cdot 2^{-\frac{m}{2\tau}} \le 2 e^{-\frac{m}{4\tau}}$.
\end{proof}

\begin{proof}[Proof of \cref{lem:dev-from-exp-cost}]
    Define
    \[
        X_k 
        =  
        \sum_{i=1}^{\Ik} c_k(s^k_i,a^k_i)
        -
        \bbE \Bigl[ \sum_{i=1}^{\Ik} c_k(s^k_i,a^k_i) \mid P,\sigma_k,s_1^k=\sinit \Bigr].
    \]
    This is a martingale difference sequence, and in order to use \cref{thm:unbounded-azuma} we need to show that $\Pr[|X_k| > m] \le 2e^{-\frac{m}{4 \tau}}$ for every $k=1,2,\dots$ and $m \ge 0$.
    This follows immediately from \cref{lem:high-prob-time} since the total cost is bounded by the total time.
    
    By \cref{thm:unbounded-azuma}, $\sum_{k=1}^K X_k \le 44 \tau \sqrt{K \log^3 \frac{4K}{\delta}}$ with probability $1 - \delta$, which gives the Lemma's statement.
\end{proof}

\subsection{Proof of \texorpdfstring{\cref{thm:hp-reg-full-info}}{}}

\begin{lemma}
    \label{lem:switch-imp-exp-cost}
    For every $k=1,\dots,K$ it holds that
    \begin{align*}
        \bbE \Bigl[ \sum_{i=1}^{\Ik} c_k(s^k_i,a^k_i) \mid P,\sigma_k,s_1^k=\sinit \Bigr]
        \le
        \bbE \Bigl[ \sum_{i=1}^{\Ik} c_k(s^k_i,a^k_i) \mid P,\pi_k,s_1^k=\sinit \Bigr]
        =
        \ctg{\pi_k}_k(\sinit).
    \end{align*}
\end{lemma}

\begin{proof}
    Until a state $s \in S$ with $\policytime{\pi_k}(s) \ge D/\cmin$ is reached, the strategy $\sigma_k$ is the same as the policy $\pi_k$.
    If such a state is reached then $\ctg{\pi_k}(s) \ge \cmin \policytime{\pi_k}(s) \ge \cmin \frac{D}{\cmin} = D$, where the first inequality is because all costs are bounded from below by $\cmin$.
    On the other hand, $\ctg{\fastpolicy}(s) \le \policytime{\fastpolicy}(s) \le D$, where the last inequality follows by the definition of the fast policy and the SSP-diameter.
    Therefore, $\ctg{\fastpolicy}(s) \le \ctg{\pi_k}(s)$.
\end{proof}

\begin{lemma}
    \label{lem:learner-strategy-bounded-time}
    For every $k=1,\dots,K$, the strategy $\sigma_k$ of the learner ensures that the expected time to the goal state from any initial state is at most $D/\cmin$.
\end{lemma}

\begin{proof}
    Let $s \in S$.
    If $\policytime{\pi_k}(s) \ge D/\cmin$, then we play the fast policy $\fastpolicy$ when we start in $s$.
    Thus, the expected time to the goal when starting in $s$ will be at most $D$.
    
    If $\policytime{\pi_k}(s) < D/\cmin$, then the expected time to the goal when starting in $s$ will also be at most $D/\cmin$ since playing $\sigma_k$ only decreases the expected time.
\end{proof}

\begin{proof}[Proof of \cref{thm:hp-reg-full-info}]
    We decompose the regret into two terms as follows,
    \begin{align*}
        \regret
        & =
        \sum_{k=1}^K \sum_{i=1}^{\Ik} c_k(s^k_i,a^k_i) 
        - 
        \sum_{k=1}^K \ctg{\pi^\star}_k (\sinit)
        \\
        & =
        \sum_{k=1}^K \sum_{i=1}^{\Ik} c_k(s^k_i,a^k_i)
        -
        \sum_{k=1}^K \bbE \Bigl[ \sum_{i=1}^{\Ik} c_k(s^k_i,a^k_i) \mid P,\sigma_k,s_1^k=\sinit \Bigr]
        \\
        & \qquad +
        \sum_{k=1}^K \bbE \Bigl[ \sum_{i=1}^{\Ik} c_k(s^k_i,a^k_i) \mid P,\sigma_k,s_1^k=\sinit \Bigr]
        -
        \sum_{k=1}^K \ctg{\pi^\star}_k(\sinit).
    \end{align*}
    The first term accounts for the deviations in the performance of the learner's strategies from their expected value, and is bounded with high probability using \cref{lem:dev-from-exp-cost}.
    
    The second term is the difference between the expected performance of the learner's strategies and the best policy in hindsight.
    Using \cref{lem:switch-imp-exp-cost}, we can bound it as follows,
    \begin{align*}
        \sum_{k=1}^K \bbE \Bigl[ \sum_{i=1}^{\Ik} c_k(s^k_i,a^k_i) \mid P,\sigma_k,s_1^k=\sinit \Bigr]
        -
        \sum_{k=1}^K \ctg{\pi^\star}_k(\sinit)
        & \le
        \sum_{k=1}^K \ctg{\pi_k}_k(\sinit)
        -
        \sum_{k=1}^K \ctg{\pi^\star}_k(\sinit)
        \\
        & =
        \sum_{k=1}^K \langle q^{\pi_k} - q^{\pi^\star} , c_k \rangle
        \\
        & \le
        \frac{2D}{\cmin} \sqrt{3 K \log \frac{D |S| |A|}{\cmin}},
    \end{align*}
    where the last inequality follows from \cref{eq:reg-olo-bound}, and the equality follows because
    \[
        \ctg{\pi}_k(\sinit)
        =
        \sum_{s \in S} \sum_{a \in A} q^\pi(s,a) c_k(s,a)
        =
        \langle q^\pi,c_k \rangle.
    \]
\end{proof}

\newpage
\section{Implementation details for SSP-O-REPS3}
\label{sec:imp-det-3}

\subsection{Computing \texorpdfstring{$q_k$}{}}

After extending the occupancy measures, we must extend our additional definitions.
Define $\KL{q}{q'}$ as the unnormalized Kullback–Leibler divergence between two occupancy measures $q$ and $q'$:
\[
    \KL{q}{q'} 
    = 
    \sum_{s \in S} \sum_{a \in A} \sum_{s' \in S^+} q(s,a,s') \log \frac{q(s,a,s')}{q'(s,a,s')} + q'(s,a,s') - q(s,a,s'),
\]
where $S^+ = S \cup \{ \ssink \}$.
Furthermore, let $R(q)$ define the unnormalized negative entropy of the occupancy measure $q$:
\[
    R(q) 
    = 
    \sum_{s \in S} \sum_{a \in A} \sum_{s' \in S^+} q(s,a,s') \log q(s,a,s') - q(s,a,s'). 
\]
SSP-O-REPS3 chooses its occupancy measures as follows:
\begin{align*}
    q_1
    & =
    q^{P_1,\pi_1} 
    = 
    \argmin_{q \in \wtboundedocset{D/\cmin}{e(1)}} R(q)
    \\
    q_{k+1}
    & =
    q^{P_{k+1},\pi_{k+1}}
    = 
    \argmin_{q \in \wtboundedocset{D/\cmin}{e(k+1)}} \eta \langle q , c_k \rangle + \KL{q}{q_k}.
\end{align*}

As shown in \cite{rosenberg2019full}, each of these steps can be split into an unconstrained minimization step, and a projection step.
Thus, $q_1$ can be computed as follows:
\begin{align*}
    q'_1
    & =
    \argmin_{q} R(q)
    \\
    q_1
    & =
    \argmin_{q \in \wtboundedocset{D/\cmin}{e(1)}} \KL{q}{q'_1},
\end{align*}
where $q'_1$ has a closed-from solution $q'_1(s,a,s') = 1$ for every $(s,a,s') \in S \times A \times S^+$.
Similarly, $q_{k+1}$ is computed as follows for every $k=1,\dots,K-1$:
\begin{align*}
    q'_{k+1}
    & =
    \argmin_{q} \eta \langle q , c_k \rangle + \KL{q}{q_k}
    \\
    q_{k+1}
    & =
    \argmin_{q \in \wtboundedocset{D/\cmin}{e(k+1)}} \KL{q}{q'_{k+1}},
\end{align*}
where again $q'_{k+1}$ has a closed-from solution $q'_{k+1}(s,a,s') = q_k(s,a,s') e^{- \eta c_k(s,a)}$ for every $(s,a,s') \in S \times A \times S^+$.

Therefore, we just need to show that the projection step can be computed efficiently (the implementation follows \cite{rosenberg2019full,jin2019learning}).
We start by formulating the projection step as a constrained convex optimization problem (where $e = e(k+1)$):
\begin{align*}
    \min_q
    &
    \quad \KL{q}{q'_{k+1}}
    &
    \\
    s.t.
    &
    \quad \sum_{a \in A} \sum_{s' \in S^+} q(s,a,s') 
    -
    \sum_{s' \in S} \sum_{a' \in A} q(s',a',s)
    =
    \indevent{s = \sinit}
    &
    \forall s \in S
    \\
    &
    \quad q(s,a,s') \le \left( \bar P_e(s' \mid s,a) + \epsilon_e(s' \mid s,a) \right) \sum_{s'' \in S^+} q(s,a,s'')
    &
    \forall (s,a,s') \in S \times A \times S^+
    \\
    &
    \quad q(s,a,s') \ge \left( \bar P_e(s' \mid s,a) - \epsilon_e(s' \mid s,a) \right) \sum_{s'' \in S^+} q(s,a,s'')
    &
    \forall (s,a,s') \in S \times A \times S^+
    \\
    &
    \quad \sum_{s \in S} \sum_{a \in A}  \sum_{s' \in S^+} q(s,a,s') \le \frac{D}{\cmin}
    &
    \\
    &
    \quad q(s,a,s') \ge 0
    &
    \forall (s,a,s') \in S \times A \times S^+
\end{align*}

To solve the problem, consider the Lagrangian:
\begin{align*}
    \mathcal{L} & (q,\lambda, v,\mu)
    =
    \KL{q}{q'_{k+1}} + \lambda \left( \sum_{s \in S} \sum_{a \in A} \sum_{s' \in S^+}q(s,a,s') - \frac{D}{\cmin} \right)
    \\
    & \quad + 
    \sum_{s \in S} v(s) \left( \sum_{s' \in S} \sum_{a' \in A} q(s',a',s) + \indevent{s = \sinit} - \sum_{a \in A} \sum_{s' \in S^+} q(s,a,s') \right)
    \\
    & \quad + 
    \sum_{s \in S} \sum_{a \in A}  \sum_{s' \in S^+} \mu^+(s,a,s') \left( q(s,a,s') - \left( \bar P_e(s' \mid s,a) + \epsilon_e(s' \mid s,a) \right) \sum_{s'' \in S^+} q(s,a,s'') \right)
    \\
    & \quad + 
    \sum_{s \in S} \sum_{a \in A}  \sum_{s' \in S^+} \mu^-(s,a,s') \left( \left( \bar P_e(s' \mid s,a) - \epsilon_e(s' \mid s,a) \right) \sum_{s'' \in S^+} q(s,a,s'') - q(s,a,s') \right)
    \\
    & =
    \KL{q}{q'_{k+1}} + v(\sinit) - \lambda \frac{D}{\cmin}
    \\
    & \quad +
    \sum_{s \in S} \sum_{a \in A} \sum_{s' \in S^+} q(s,a,s') \Biggl( \lambda + v(s') - v(s) + \mu^+(s,a,s') - \mu^-(s,a,s')
    \\
    & \qquad \qquad \qquad \qquad \qquad \qquad \quad -
    \sum_{s'' \in S^+} \bar P_e(s'' \mid s,a) (\mu^+(s,a,s'') - \mu^-(s,a,s''))
    \\
    & \qquad \qquad \qquad \qquad \qquad \qquad \quad -
    \sum_{s'' \in S^+} \epsilon_e(s'' \mid s,a) (\mu^+(s,a,s'') + \mu^-(s,a,s'')) \Biggr)
\end{align*}
where $\lambda$, $\{ v(s) \}_{s \in S}$, $\{ \mu^+(s,a,s')\}_{(s,a,s') \in S \times A \times S^+}$ and $\{ \mu^-(s,a,s')\}_{(s,a,s') \in S \times A \times S^+}$ are Lagrange multipliers, and we set $v(\ssink) = 0$ for convenience.
Differentiating the Lagrangian with respect to any $q(s, a,s')$, we get
\begin{align*}
    \frac{\partial \mathcal{L} (q,\lambda,v,\mu)}{\partial q(s,a,s')}
    & =
    \log \frac{q(s,a,s')}{q'_{k+1}(s,a,s')} + \lambda + v(s') - v(s) + \mu^+(s,a,s') - \mu^-(s,a,s')
    \\
    & \quad -
    \sum_{s'' \in S^+} \bar P_e(s'' \mid s,a) (\mu^+(s,a,s'') - \mu^-(s,a,s''))
    \\
    & \quad -
    \sum_{s'' \in S^+} \epsilon_e(s'' \mid s,a) (\mu^+(s,a,s'') + \mu^-(s,a,s'')).
\end{align*}

Next we define
\begin{align}
    \nonumber
    B_k^{v,\mu}(s,a,s')
    & =
    v(s) - v(s') + \mu^-(s,a,s') - \mu^+(s,a,s') - \eta c_k(s,a)
    \\
    \nonumber
    & \quad +
    \sum_{s'' \in S^+} \bar P_{e(k+1)}(s'' \mid s,a) (\mu^+(s,a,s'') - \mu^-(s,a,s''))
    \\
    \label{eq:B-k-def}
    & \quad +
    \sum_{s'' \in S^+} \epsilon_{e(k+1)}(s'' \mid s,a) (\mu^+(s,a,s'') + \mu^-(s,a,s'')).
\end{align}

Hence, setting the gradient to zero, we obtain the formula for $q_{k+1}(s, a)$:
\begin{align}
    \nonumber
    q_{k+1}(s,a,s')
    & =
    q'_{k+1}(s,a,s') e^{-\lambda + \eta c_k(s,a) + B_k^{v,\mu}(s,a,s')}
    \\
    \label{eq:q-k-update-unknown}
    & =
    q_k(s,a,s') e^{-\lambda + B_k^{v,\mu}(s,a,s')},
\end{align}
where the last equality follows from the formula of $q'_{k+1}(s,a,s')$, and setting $c_0(s,a) = 0$ and $q_0(s,a,s') = 1$ for every $(s,a,s') \in S \times A \times S^+$.

We now need to compute the value of $\lambda,v,\mu$ at the optimum.
To that end, we write the dual problem $\mathcal{D}(\lambda, v,\mu) = \min_q \mathcal{L}(q,\lambda,v,\mu)$ by substituting $q_{k+1}$ back into $\mathcal{L}$:
\begin{align*}
    \mathcal{D}(\lambda, v,\mu)
    & =
    \sum_{s \in S} \sum_{a \in A} \sum_{s' \in S^+} q'_{k+1}(s,a,s') - \sum_{s \in S} \sum_{a \in A} \sum_{s' \in S^+} q_{k+1}(s,a,s') + v(\sinit) - \lambda \frac{D}{\cmin}
    \\
    & = 
    - \sum_{s \in S} \sum_{a \in A} \sum_{s' \in S^+} q_k(s,a,s') e^{-\lambda + B_k^{v,\mu}(s,a,s')}
    +
    v(\sinit) - \lambda \frac{D}{\cmin} + \sum_{s \in S} \sum_{a \in A} \sum_{s' \in S^+} q'_{k+1}(s,a,s').
\end{align*}

Now we obtain $\lambda,v,\mu$ by maximizing the dual.
Equivalently, we can minimize the negation of the dual (and ignore the term $\sum_{s \in S} \sum_{a \in A} \sum_{s' \in S^+} q'_{k+1}(s,a,s')$), that is:
\begin{align*}
    \lambda_{k+1}, v_{k+1}, \mu_{k+1}
    = 
    \argmin_{\lambda \ge 0, v, \mu \ge 0} \sum_{s \in S} \sum_{a \in A} \sum_{s' \in S^+} q_k(s,a,s') e^{-\lambda + B_k^{v,\mu}(s,a,s')} + \lambda \frac{D}{\cmin} - v(\sinit).
\end{align*}
This is  a convex optimization problem with only non-negativity constraints (and no constraints about the relations between the variables), which can be solved efficiently using iterative methods like gradient descent.

\subsection{Computing the optimistic fast policy}
\label{sec:comp-optimistic-D}

The optimistic fast policy $\optimisticfastpolicy{e}$ is a deterministic stationary policy that together with the optimistic fast transition function from the confidence set of epoch $e$, minimizes the time to the goal state from all states simultaneously out of all pairs of policies and transition functions from the confidence set.
Essentially, this is the optimal pair of policy and transition function from the confidence set w.r.t the constant cost function $c(s,a) = 1$ for every $s \in S$ and $a \in A$.

The existence of the optimistic fast policy is proven in \cite{tarbouriech2019noregret}, and there they also show that it can be computed efficiently with Extended Value Iteration.
In \cite{cohen2020ssp}, the authors compute the following optimistic fast transition function for every $(s,a,s') \in S \times A \times S$:
\[
    \wt P^f_e(s' \mid s,a)
    =
    \max \left\{ 0 , \bar P_e(s' \mid s,a) - 28 A^e(s,a) - 4 \sqrt{\bar P_e(s' \mid s,a) A^e(s,a)} \right\},
\]
where the remaining probability mass goes to $\wt P^f_e(\ssink \mid s,a)$.
Then, $\optimisticfastpolicy{e}$ is computed by finding the fast policy w.r.t $\wt P^f_e$ (see \cref{sec:comp-D}).

While this method is simpler and more efficient than Extended Value Iteration, the authors do not prove that this is indeed the optimistic fast policy.
However, this policy is sufficient for their analysis and for our analysis as well.
For simplicity, throughout the analysis we assume that $\optimisticfastpolicy{e}$ is the optimistic fast policy, but every step of the proof works with this computation as well.

\newpage
\section{Pseudo-code for SSP-O-REPS3}
\label{sec:code-3}

\begin{algorithm}
    \caption{\sc SSP-O-REPS3}
    \label{alg:alg-3}
    \begin{algorithmic} 
        
        \STATE {\bfseries Input:} state and space $S$, action space $A$, minimal cost $\cmin$, optimization parameter $\eta$ and confidence parameter $\delta$.

        \STATE {\bfseries Initialization:}
        
        \STATE Obtain SSP-diameter $D$ from user or estimate it (see \cref{sec:estimate-diameter-proofs}).
         
        \STATE Set $q_0(s,a,s') = 1$ and $c_0(s,a) = 0$ for every $(s,a,s') \in S \times A \times S^+$.
        
        \STATE Set $e \gets 0$ and for every $(s,a,s') \in S \times A \times S^+$: $N^0(s,a) \gets 0,N^0(s,a,s') \gets 0,n^0(s,a) \gets 0,n^0(s,a,s') \gets 0$.

        \smallskip
         
        \FOR{$k=1,2,\ldots$}
        
        \STATE $e \gets e+1$, start new epoch (\cref{alg:alg-4}).
         
         \smallskip
         
         \STATE Set $s_1^k \gets \sinit$, $i \gets 1$.
        
        \WHILE{$s_i^k \neq \ssink$ and $\optimisticpolicytime{\pi_k}_k(s_i^k) < \frac{D}{\cmin}$ and $\forall a \in A. \,  n^e(s_i^k,a) + N^e(s_i^k,a) > \alpha \numvisitsuntilknownbern$}
        
        \STATE Play action according to $\pi_k$, i.e., $a_i^k \sim \pi_k(\cdot \mid s_i^k)$.
        
        \STATE Observe next state $s_{i+1}^k \sim P(\cdot \mid s_i^k,a_i^k)$.
        
        \STATE Update counters: $n^e(s_i^k,a_i^k) \gets n^e(s_i^k,a_i^k)+1$,$n^e(s_i^k,a_i^k,s_{i+1}^k) \gets n^e(s_i^k,a_i^k,s_{i+1}^k)+1$.
        
        \STATE $i \gets i+1$.
        
        \IF{$n^e(s_{i-1}^k,a_{i-1}^k) \ge N^e(s_{i-1}^k,a_{i-1}^k)$}
        
            \STATE $e \gets e+1$, start new epoch (\cref{alg:alg-4}).
        
            \STATE break
        
        \ENDIF

        \ENDWHILE
        
        \smallskip
        
        \WHILE{$s_i^k \neq \ssink$}
        
        \IF{$\exists a \in A. \,  n^e(s_i^k,a) + N^e(s_i^k,a) \le \alpha \numvisitsuntilknownbern$}
        
        \STATE Play the least played action $a_i^k = \argmin_{a \in A} n^e(s_i^k,a) + N^e(s_i^k,a)$.
        
        \ELSE
        
        \STATE Play according to $\optimisticfastpolicy{e}$, i.e., $a_i^k \sim \optimisticfastpolicy{e}(\cdot \mid s_i^k)$.
        
        \ENDIF
        
         \STATE Observe next state $s_{i+1}^k \sim P(\cdot \mid s_i^k,a_i^k)$.
        
        \STATE Update counters: $n^e(s_i^k,a_i^k) \gets n^e(s_i^k,a_i^k)+1$,$n^e(s_i^k,a_i^k,s_{i+1}^k) \gets n^e(s_i^k,a_i^k,s_{i+1}^k)+1$.
        
        \STATE $i \gets i+1$.
        
        \IF{$n^e(s_{i-1}^k,a_{i-1}^k) \ge N^e(s_{i-1}^k,a_{i-1}^k)$}
        
        \STATE $e \gets e+1$, start new epoch (\cref{alg:alg-4}).
        
        \ENDIF

        \ENDWHILE
        
        \smallskip
        
        \STATE Set $\Ik \gets i-1$.
        
        \STATE Observe cost function $c_k$ and suffer cost $\sum_{j=1}^{I_k}c_k(s_j^k,a_j^k)$.
        
        \ENDFOR
    \end{algorithmic}
\end{algorithm}

\newpage

\begin{algorithm}
    \caption{\sc Start new epoch}
    \label{alg:alg-4}
    \begin{algorithmic} 
    
    \STATE Update counters for every $(s,a,s') \in S \times A \times S^+$:
    \begin{align*}
        N^e(s,a) \gets N^{e-1}(s,a) + n^{e-1}(s,a)
        & \quad ; \quad
        n^e(s,a) \gets 0
        \\
        N^e(s,a,s') \gets N^{e-1}(s,a,s') + n^{e-1}(s,a,s')
        & \quad ; \quad
        n^e(s,a,s') \gets 0
    \end{align*}
    
    \STATE Update confidence set for every $(s,a,s') \in S \times A \times S^+$:
    \begin{align*}
        \bar P_e(s' \mid s,a)
        & =
        \frac{N^e(s,a,s')}{N_+^e(s,a)}
        \\
        \epsilon_e(s' \mid s,a)
        & =
        4 \sqrt{\bar P_e (s' \mid s,a) A^e(s,a)}
        +
        28 A^e(s,a),
    \end{align*}
    where $A^e(s,a) = \frac{\log (|S| |A| N^e_+(s,a) / \delta)}{N^e_+(s,a)}$.
    
    \IF{$e$ is the first epoch of episode $k$}
    
    \STATE Compute $\lambda_k,v_k,\mu_k$ as follows (using, e.g., gradient descent):
         \[
            \lambda_k, v_k, \mu_k
            = 
            \argmin_{\lambda \ge 0, v, \mu \ge 0} \sum_{s \in S} \sum_{a \in A} \sum_{s' \in S^+} q_{k-1}(s,a,s') e^{-\lambda + B_{k-1}^{v,\mu}(s,a,s')} + \lambda \frac{D}{\cmin} - v(\sinit),
         \]
         where $B_k^{v,\mu}(s,a,s')$ is defined in \cref{eq:B-k-def}.
         
         \STATE Compute $q_k$ as follows for every $(s,a,s') \in S \times A \times S^+$:
         \[
            q_k(s,a,s')
            =
           q_{k-1}(s,a,s') e^{-\lambda_k +B_{k-1}^{v_k,\mu_k}(s,a,s')}. 
         \]
         
         \STATE Compute $\pi_k$ and $P_k$ as follows for every $(s,a,s') \in S \times A \times S^+$:
         \[
            \pi_k(a \mid s)
            =
            \frac{\sum_{s' \in S^+} q_k(s,a,s')}{\sum_{a' \in A} \sum_{s' \in S^+} q_k(s,a',s') }
            \quad ; \quad
            P_k(s' \mid s,a)
            =
            \frac{ q_k(s,a,s')}{\sum_{s'' \in S^+} q_k(s,a,s'') }
         \]
         
        \STATE Set $\optimisticpolicytime{\pi_k}_k(s) \gets \frac{D}{\cmin}$ for every $s \in S$ such that $\sum_{a \in A} \sum_{s' \in S^+} q_k(s,a,s') = 0$.

        \STATE Compute $\optimisticpolicytime{\pi_k}_k$ by solving the following linear equations:
         \[
            \optimisticpolicytime{\pi_k}_k (s) 
            = 
            1 + \sum_{a \in A} \sum_{s' \in S} \pi_k(a | s) P_k(s' | s,a) \optimisticpolicytime{\pi_k}_k(s')
            \quad \forall s \in \{ s \in S : \sum_{a \in A} \sum_{s' \in S^+} q_k(s,a,s') >0 \}.
         \]
         
    \ELSE
    
    \STATE Compute the optimistic fast policy $\optimisticfastpolicy{e}$ (see \cref{sec:comp-optimistic-D}).
    
    \ENDIF
         
    \end{algorithmic}
\end{algorithm}

\newpage
\section{Analysis of SSP-O-REPS3 (proofs for \texorpdfstring{\cref{sec:high-prob-unknown-P}}{})}
\label{sec:unknown-P-proofs}

\subsection{Overview}

Our analysis follows the framework of \cite{cohen2020ssp} for analyzing optimism in SSPs, but makes the crucial adaptations needed to handle the adversarial environment.

We have two objectives: bounding the number of steps $\totaltime$ taken by the algorithm (to show that we reach the goal in every episode) and bounding the regret.
To bound the total time we split the time steps into \emph{intervals}.
The first interval begins at the first time step, and an interval ends once (1) an episode ends, (2) an epoch ends, (3) an unknown state is reached, or (4) a state $s$ such that $\optimisticpolicytime{\pi_k}_k(s) \ge D/\cmin$ is reached when playing $\pi_k$ in episode $k$, i.e., there is a switch.

Intuitively, we will bound the length of every interval by $\tO(D/\cmin)$ with high probability, and then use the number of intervals $\tO(K + D|S|^2|A|/\cmin^2)$ to bound the total time $T$.
Finally, we will show that the regret scales with the square root of the total variance (which is the number of intervals times the variance in each interval) to finish the proof.
While intuitive, this approach is technically difficult and therefore we apply these principles in a different way.

We start by showing that the confidence sets contain $P$ with high probability, which is a common result (see, e.g., \cite{zanette2019tighter,efroni2019tight}).
Define $\geventi{m}$ the event that $P$ is in the confidence set of the epoch that interval $m$ belongs to.

\begin{lemma}[\cite{cohen2020ssp}, Lemma 4.2]
    \label{lem:conf-set-hold}
    With probability at least $1 - \delta/2$, the event $\geventi{m}$ holds for all intervals $m$ simultaneously.
\end{lemma}

There are two dependant probabilistic events that are important for the analysis.
The first are the events $\geventi{m}$, and the second is that the deviation in the cost of a given policy from its expected value is not large.
To disentangle these events we define an alternative regret for every $\numintervals = 1,2,\dots$,
\[
    \tregret{\numintervals}
    =
    \sum_{m=1}^{\numintervals} \sum_{h=1}^{\Hm} \sum_{a \in A} \optimisticpi{m} (a \mid s_h^m) c_m(s_h^m,a) \indgeventi{m}
    -
    \sum_{k=1}^K \ctg{\pi^\star}_k(\sinit),
\]
where $c_m=c_k$ for the episode $k$ that interval $m$ belongs to, $\optimisticpi{m}$ is the policy followed by the learner in interval $m$, $\Hm$ is the length of interval $m$, and the trajectory visited in interval $m$ is $\traj{m} = ( s_1^m, a_1^m , \ldots , s_{\Hm}^m, a_{\Hm}^m, s_{\Hm+1}^m)$.

We focus on bounding $\tregret{\numintervals}$ because we can use it to obtain a bound on $\regret$.
This is done using \cref{lem:conf-set-hold} and an application of Azuma inequality, when $\numintervals$ is the number of intervals in which the first $K$ episodes elapse (we show that the learner indeed completes these $K$ episodes).

As mentioned, bounding the length of each interval complicates the analysis, and therefore we introduce artificial intervals.
That is, an interval $m$ also ends at the first time step $H$ such that $\sum_{h=1}^{H} \sum_{a \in A} \optimisticpi{m} (a \mid s_h^m) c_m(s_h^m,a) \ge D/\cmin$.
The artificial intervals are only introduced for the analysis and do not affect the algorithm.
Now, the length of each interval is bounded by $2D/\cmin^2$ and we can bound the number of intervals as follows.

\begin{lemma}
    \label{obs:cost-bounds-intervals}
    Let $\wt C_M = \sum_{m=1}^{\numintervals} \sum_{h=1}^{\Hm} \sum_{a \in A} \optimisticpi{m} (a \mid s_h^m) c_m(s_h^m,a)$.
    The total time satisfies $\totaltime \le \wt C_M / \cmin$ and the total number of intervals satisfies
    \[
        M
        \le
        \frac{\cmin \wt C_M}{D} + 2 |S| |A| \log T + 2K + 2\alpha \frac{D |S|^2 |A|}{\cmin^2} \log \frac{D |S| |A|}{\delta \cmin}.
    \]
\end{lemma}

Note that a confidence set update occurs only in the end of an epoch and thus $\geventi{m} = \geventi{m-1}$ for most intervals.
Also, for artificial intervals the policy does not change.
Next we bound $\wt C_M$ as a function of the number of intervals $M$.
Through summation of our confidence bounds, and by showing that the variance in each interval is bounded by $D^2 / \cmin^2$ we are able to obtain the following, when \cref{lem:conf-set-hold} holds,
\[
    \wt C_M
    \le
    \sum_{k=1}^K \langle q_k , c_k \rangle 
    + \tO \biggl(
    \frac{D|S|}{\cmin} \sqrt{\numintervals |A|} 
    + 
    \frac{D^2 |S|^2 |A|}{\cmin^2} \biggr).
\]
Substituting in \cref{obs:cost-bounds-intervals} and solving for $\wt C_M$ we get
\begin{align*}
    \wt R_M
    & =
    \wt C_M
    -
    \sum_{k=1}^K \ctg{\pi^\star}_k(\sinit)
    \le
    \sum_{k=1}^K \langle q_k - q^{P,\pi^\star} , c_k \rangle
    \\
    & \qquad + \tO \biggl(
    \frac{D|S|}{\cmin} \sqrt{|A| K} 
    + 
    \frac{D^2 |S|^2 |A|}{\cmin^2} \biggr),
\end{align*}
Notice that the first term on the RHS of the inequality is exactly the regret of OMD, and therefore analyzing it similarly to \cref{thm:exp-reg-full-info} gives the final bound (see \cref{sec:omd-unknown-P-proofs}).

\subsection{Notations}
\label{sec:bern-analysis-notations}

Denote the trajectory visited in interval $m$ by $\traj{m} = ( s_1^m, a_1^m , \ldots , s_{\Hm}^m, a_{\Hm}^m, s_{\Hm+1}^m)$, where $a_h^m$ is the action taken in $s_h^m$, and $\Hm$ is the length of the interval.
In addition, the concatenation of trajectories in the intervals up to and including interval $m$ is denoted by $\trajconcat{m}$, that is 
$\trajconcat{m} = \cup_{m'=1}^m \traj{m'}$.

The policy that the learner follows in interval $m$ is denoted by $\optimisticpi{m}$, and the transition function that was involved in the choice of $\optimisticpi{m}$ is denoted by $\optimisticP{m}$.
For the first interval of every episode these are chosen by OMD, i.e., $\pi_k$ and $P_k$, and for other intervals these are the optimistic fast policy $\optimisticfastpolicy{e}$ and the transition function chosen from the confidence set together with it $\wt P^f_e$, for the epoch $e$ that interval $m$ belongs to.
Notice that intervals with unknown states are of length $1$. Thus, there is no policy since only one action is performed -- we ignore visits to unknown states and we suffer their cost directly in \cref{lem:bern-reg-decomp}.

The expected cost of $\optimisticpi{m}$ w.r.t $\optimisticP{m}$ is denoted by $\optimisticctg{m}$, and the expected time to the goal is denoted by $\optimisticpolicytime{m}$.
For intervals in which we follow the optimistic fast policy, we will show that $\optimisticpolicytime{m}(s) \le D$ for every $s \in S$ when $\geventi{m}$ holds.
We would like to have a similar property for intervals in which we follow the OMD policy, i.e., the first interval of every episode.

Note that for the first interval $m$ of episode $k$, we have that $\optimisticpolicytime{\pi_k}_k = \optimisticpolicytime{m}$, and recall that reaching a state $s \in S$ such that $\optimisticpolicytime{\pi_k}_k(s) \ge D/\cmin$ ends the current interval.
We would like to take advantage of this fact in order to make sure that $\optimisticpolicytime{m}$ is always bounded by $D/\cmin$.
Similarly to \cref{sec:high-prob-known-P}, we compute $\optimisticpolicytime{\pi_k}_k(s)$ only for states $s$ that are reachable from $\sinit$ w.r.t $P_k$.
Since reaching a state $s$ with $\optimisticpolicytime{\pi_k}_k(s) \ge D/\cmin$ yields the start of a new interval for which we use the optimistic fast policy, we can set $\optimisticpolicytime{\pi_k}_k(s) = D/\cmin$ for states that are not reachable from $\sinit$ without affecting the algorithm's choices.

We make another change to $\optimisticP{m}$ for interval $m$ that is the first interval of episode $k$.
Since reaching a state $s \in S$ such that $\optimisticpolicytime{\pi_k}_k(s) \ge D/\cmin$ ends the interval, we tweak $\optimisticP{m}$ such that from such a state it goes directly to the goal with expected time of $D/\cmin$ and expected cost of $D$ (can be done with a self-loop that has $\cmin/D$ probability to go to $\ssink$).
Thus, when we consider the expected cost of $\optimisticpi{m}$ w.r.t $\optimisticP{m}$, we have that $\optimisticctg{m}(\sinit) \le \optimisticctg{\pi_k}_k(\sinit)$ because we only decreased the cost from some states.
However, notice that now $\optimisticP{m}$ is in the confidence set only for states that we did not tweak.
We show that this does not affect the analysis, since reaching those states ends the interval.

We would like to emphasize that tweaking $\optimisticP{m}$ is only done in hindsight as a part of the analysis, and does not change the algorithm.

\subsection{Properties of the learner's policies}

\begin{lemma}
    \label{lem:bern-opt-val-bound}
    Let $m$ be an interval.
    If $m$ is the first interval of episode $k$ then $\optimisticpolicytime{m}(s) \le D/\cmin$ for every $s \in S$.
    Otherwise, if $\geventi{m}$ holds then $\optimisticpolicytime{m}(s) \le D$ for every $s \in S$.
\end{lemma}

\begin{proof}
   The first case holds by definition of $\optimisticP{m}$ for intervals that are in the beginning of some episode (see discussion in \cref{sec:bern-analysis-notations}).
   The second case follows by optimism and the fact that $P$ is in the confidence set (see \cite{cohen2020ssp}, Lemma B.2).
\end{proof}

\begin{lemma}
    \label{lem:bellman-optimistic}
    Let $m$ be an interval and let $1 \le h \le \Hm$. 
    If $\geventi{m}$ holds then the following Bellman equations hold:
    \begin{align*}
        \optimisticctg{m}(s_h^m)
        & =
        \sum_{a \in A} \optimisticpi{m}(a \mid s_h^m) c_m(s_h^m,a) + \sum_{a \in A} \sum_{s' \in S} \optimisticpi{m}(a \mid s_h^m) \optimisticP{m}(s' \mid s_h^m,a) \optimisticctg{m}(s')
        \\
        \optimisticpolicytime{m}(s_h^m)
        & =
        1 + \sum_{a \in A} \sum_{s' \in S} \optimisticpi{m}(a \mid s_h^m) \optimisticP{m}(s' \mid s_h^m,a) \optimisticpolicytime{m}(s').
    \end{align*}
\end{lemma}

\begin{proof}
    For the optimistic fast policy $\optimisticfastpolicy{e}$ the Bellman equations hold for every $s \in S$ since it is proper w.r.t $\wt P^f_e$ (see \cite{cohen2020ssp}, Lemma B.11).
    When $\optimisticpi{m}$ is the policy chosen by OMD $\pi_k$, reaching a state $s$ such that $q^{P_k,\pi_k}(s) = 0$ will end the interval (since we set $\optimisticpolicytime{\pi_k}_k(s) = D/\cmin$ for these states).
    Thus, it suffices to show that the Bellman equations hold for all states in $\{ s \in S : q^{P_k,\pi_k}(s) > 0 \}$.
    
    For these states we have that $\optimisticpolicytime{m}$ is bounded by $D/\cmin$ and therefore $\optimisticpi{m}$ is proper w.r.t $\optimisticP{m}$ and the Bellman equations hold.
    Note that we did not make changes to $\optimisticP{m}$ or $c_m$ in states that can be visited during the interval.
\end{proof}

\subsection{Regret decomposition}

\begin{lemma}
    \label{lem:bern-reg-decomp}
    It holds that
    \[
        \tregret{M} 
        \le 
        \sum_{m=1}^\numintervals \tregret{m}^1 + \sum_{m=1}^\numintervals \tregret{m}^2 - \sum_{k=1}^K \ctg{\pi^\star}_k(\sinit) + \alpha \frac{D |S|^2 |A|}{\cmin^2} \log \frac{D |S| |A|}{\delta \cmin},
    \]
    where
    \begin{align*}
        \tregret{m}^1 
        & = 
        \bigl( \optimisticctg{m}(s_1^m) - \optimisticctg{m}(s_{\Hm+1}^m) \bigr) \indgeventi{m} 
        \\
        \tregret{m}^2 
        & = 
        \sum_{h=1}^{\Hm} \Biggl( \optimisticctg{m}(s_{h+1}^m) - \sum_{a \in A} \sum_{s' \in S} \optimisticpi{m}(a \mid s_h^m) \optimisticP{m}(s' \mid s_h^m,a) \optimisticctg{m}(s') \Biggr) \indgeventi{m}.
    \end{align*}
\end{lemma}

\begin{proof}
    First we have a cost of at most $1$ every time we visit an unknown state.
    Each state becomes known after $\alpha |A| \numvisitsuntilknownbern$ visits, and therefore the total cost from these visits is at most $\alpha |S| |A| \numvisitsuntilknownbern$.
    From now on we will ignore visits to unknown states throughout the analysis because we calculated their contribution to the total cost.
    
    We can use the Bellman equations w.r.t $\optimisticP{m}$ (\cref{lem:bellman-optimistic}) to have the following interpretation of the costs for every interval $m$ and time $h$:
    \begin{align}
        \nonumber
        \sum_{a \in A} \optimisticpi{m}(a \mid s_h^m) & c_m(s_h^m, a) \indgeventi{m}
        = 
        \\
        \nonumber
        & =
        \Biggl( \optimisticctg{m}(s_h^m) - \sum_{a \in A} \sum_{s' \in S} \optimisticpi{m}(a \mid s_h^m) \optimisticP{m}(s' \mid s_h^m,a) \optimisticctg{m}(s') \Biggr) \indgeventi{m} 
        \\
        \nonumber
        &= 
        \Biggl( \optimisticctg{m}(s_h^m) -  \optimisticctg{m}(s_{h+1}^m) \Biggr) \indgeventi{m}
        \\
        \label{eq:cost-formula}
        & \quad + 
        \Biggl( \optimisticctg{m}(s_{h+1}^m) - \sum_{a \in A} \sum_{s' \in S} \optimisticpi{m}(a \mid s_h^m) \optimisticP{m}(s' \mid s_h^m,a) \optimisticctg{m}(s') \Biggr) \indgeventi{m}.
    \end{align}
    We now write 
    $
        \tregret{M}
        =
        \sum_{m=1}^\numintervals \sum_{h=1}^{\Hm} \sum_{a \in A} \optimisticpi{m}(a \mid s_h^m) c_m(s_h^m, a) \indgeventi{m} 
        - 
        \sum_{k=1}^K \ctg{\pi^\star}_k(\sinit),
    $ 
    and substitute for each cost using \cref{eq:cost-formula} to get the lemma, noting that the first term telescopes within the interval.
\end{proof}

\begin{lemma}
    \label{lem:bern-first-reg-term-bound}
    It holds that
    \[
        \sum_{m=1}^\numintervals \tregret{m}^1 \le 
        2 D |S| |A| \log \totaltime + \alpha \frac{D^2 |S|^2 |A|}{\cmin^2} \log \frac{D |S| |A|}{\delta \cmin} + \sum_{k=1}^K \optimisticctg{\pi_k}_k(\sinit) \indgeventi{m(k)},
    \]
    where $m(k)$ is the first interval of episode $k$.
\end{lemma}

\begin{proof}
    For every two consecutive intervals $m,m+1$ we have one of the following:
    \begin{enumerate}
        \item If interval $m$ ended in the goal state then
        $
            \optimisticctg{m}(s_{\Hm +1}^m) 
            =
            \optimisticctg{m}(\ssink) = 0
        $ 
        and
        $
            \optimisticctg{m+1}(s_1^{m+1})
            =
            \optimisticctg{m(k)}(\sinit) 
            \le
            \optimisticctg{\pi_k}_k(\sinit)
        $, where $m+1$ is the first interval of episode $k$.
        Therefore,
        \[
            \optimisticctg{m+1}(s_1^{m+1}) \indgeventi{m+1}  
            - 
            \optimisticctg{m}(s_{\Hm + 1}^m) \indgeventi{m} 
            \le
            \optimisticctg{\pi_k}_k(\sinit) \indgeventi{m(k)}.
        \]
        This happens at most $K$ times, once for every value $k$.
        
        \item If interval $m$ ended since the sum of expected costs in the interval passed $D/\cmin$, then we did not change policy.
        Thus, $\optimisticctg{m} = \optimisticctg{m+1}$, $\geventi{m}=\geventi{m+1}$ and $s_1^{m+1} = s_{\Hm +1}^m$. We get
        \[
            \optimisticctg{m+1}(s_1^{m+1}) \indgeventi{m+1}  
            - 
            \optimisticctg{m}(s_{\Hm +1}^m) \indgeventi{m}
            =
            0.
        \]
        
        \item If interval $m$ ended by reaching an unknown state, then we switch policy. Thus,
        \[
            \optimisticctg{m+1}(s_1^{m+1}) \indgeventi{m+1}  
            - 
            \optimisticctg{m}(s_{\Hm +1}^m) \indgeventi{m}
            \le 
            \optimisticctg{m+1}(s_1^{m+1}) \indgeventi{m+1} 
            \le 
            D,
        \]
        where the last inequality follows because we switched to the optimistic fast policy and thus its expected time will be bounded by $D$ if $P$ is in the confidence set (see \cref{lem:bern-opt-val-bound}).
        This happens at most $|S| |A| \alpha \numvisitsuntilknownbern$ times.
        
        Here we ignored the unknown state (since we accounted for its cost in \cref{lem:bern-reg-decomp}) and jumped right to the next interval, which is controlled by the optimistic fast policy.
        
        \item If interval $m$ ended with doubling the visits to some state-action pair, then similarly to the previous article,
        \[
            \optimisticctg{m+1}(s_1^{m+1}) \indgeventi{m+1}  
            - 
            \optimisticctg{m}(s_{\Hm +1}^m) \indgeventi{m}
            \le 
            \optimisticctg{m+1}(s_1^{m+1}) \indgeventi{m+1} 
            \le 
            D.
        \]
        This happens at most $2 |S| |A| \log \totaltime$.
        
        \item If $m$ is the first interval of an episode $k$ and it ended because we reached a ``bad'' state then $\optimisticctg{m}(s^m_{\Hm+1}) = D$ and $\optimisticctg{m+1}(s^{m+1}_1) \le D$ since this is the optimistic fast policy. Thus,
        \[
            \optimisticctg{m+1}(s_1^{m+1}) \indgeventi{m+1}  
            - 
            \optimisticctg{m}(s_{\Hm +1}^m) \indgeventi{m}
            \le
            0.
        \]
    \end{enumerate}
\end{proof}

\begin{lemma}
    \label{lem:optimistic-val-func-diff-to-expected}
    With probability at least $1 - \delta / 6$,  the following holds for all $M = 1,2,\ldots$ simultaneously.
    \begin{align*}
        \sum_{m=1}^\numintervals \tregret{m}^2
        \le
        \sum_{m=1}^\numintervals \bbE \bigl[ \tregret{m}^2 \mid \trajconcat{m-1} \bigr]
        +
        \frac{6 D}{\cmin} \sqrt{\numintervals \log \frac{4 \numintervals}{\delta}},
    \end{align*}
    where $\bbE [\cdot \mid \trajconcat{m-1}]$ is the expectation conditioned on the trajectories up to interval $m$.
\end{lemma}

\begin{proof}
    Consider the martingale difference sequence $(Y^m)_{m=1}^\infty$ defined by $Y^m = X^m - \bbE[X^m \mid \trajconcat{m-1}]$ and
    \[
        X^m 
        = 
        \sum_{h=1}^{\Hm} \Biggl( \optimisticctg{m}(s_{h+1}^m) - \sum_{a \in A} \sum_{s' \in S} \optimisticpi{m}(a \mid s_h^m) \optimisticP{m}(s' \mid s_h^m,a) \optimisticctg{m}(s') \Biggr) \indgeventi{m}.
    \]
    The Bellman equations of $\optimisticpi{m}$ w.r.t $\optimisticP{m}$ (\cref{lem:bellman-optimistic}) obtain
    \begin{align*}
        |X^m|
        & = 
        \biggl| \biggl( \underbrace{\optimisticctg{m}(s_{\Hm +1}^m) - \optimisticctg{m}(s_{1}^m)}_{\le D/\cmin} + 
        \\
        & \qquad +
        \underbrace{\sum_{h=1}^{\Hm} \optimisticctg{m}(s_h^m) - \sum_{a \in A} \sum_{s' \in S} \optimisticpi{m}(a \mid s_h^m) \optimisticP{m}(s' \mid s_h^m,a) \optimisticctg{m}(s')}_{= \sum_{h=1}^{\Hm} \sum_{a \in A} \optimisticpi{m}(a \mid s_h^m) c_m(s_h^m,a)} \biggr) \indgeventi{m} \biggr|
        \\
        & \le
        \frac{D}{\cmin} + \sum_{h=1}^{\Hm} \sum_{a \in A} \optimisticpi{m}(a \mid s_h^m) c_m(s_h^m,a)
        \le
        \frac{3D}{\cmin}
    \end{align*}
    where for the first inequality we used \cref{lem:bern-opt-val-bound,lem:bellman-optimistic}, and the last inequality follows because the cost in every interval is at most $2D/\cmin$.
    
    Therefore, we use anytime Azuma inequality (\cref{thm:anytime-azuma}) to obtain that with probability at least $1-\delta / 6$:
    \begin{align*}
        \sum_{m=1}^M X^m 
        \le
        \sum_{m=1}^M \bbE \bigl[ X^m \mid \trajconcat{m-1} \bigr] 
        +
        \frac{6 D}{\cmin} \sqrt{M \log \frac{4 M}{\delta}}. \qquad 
        \label{eq:bern-azuma-deviations}
    \end{align*}
\end{proof}

\subsection{Bounding the variance within an interval}

\begin{lemma}[\cite{cohen2020ssp}, Lemma B.13]
    \label{lem:relbernstein} 
    Denote $A^m(s,a) = \frac{\log (|S| |A| N_+^{e(m)}(s,a) / \delta)}{ N_+^{e(m)}(s,a)}$, where $e(m)$ is the epoch that interval $m$ belongs to.
    When $\geventi{m}$ holds we have for any $(s,a,s') \in S \times A \times S^+$:
    \[
        \bigl| P (s' \mid s,a) - \optimisticP{m} (s' \mid s,a) \bigr|
        \le
        8 \sqrt{P(s' \mid s,a) A^m(s,a)}
        +
        136 A^m(s,a).
    \]
\end{lemma}

\begin{lemma}
    \label{lem:sum-bern-bounds}
    Denote $A_h^m = A^m(s_h^m,a_h^m)$.
    For every interval $m$ it holds that,
    \begin{align*}
        \bbE [ \tregret{m}^2 \mid \trajconcat{m-1} ] 
        & \le 
        16 \bbE \Biggl[ 
        \sum_{h=1}^{\Hm} \sqrt{|S| \bbV_h^m A_h^m} \indgeventi{m} 
        \biggm| \trajconcat{m-1} \Biggr]
        + 
        272 \bbE \Biggl[ \sum_{h=1}^{\Hm} \frac{D}{\cmin} |S| A_h^m \indgeventi{m} 
        \biggm| \trajconcat{m-1} \Biggr],
    \end{align*}
    where $\bbV_h^m$ is the empirical variance defined as
    \[
        \bbV_h^m 
        = 
        \sum_{s' \in S^+} P(s' \mid s_h^m,a_h^m) \Biggl(\optimisticctg{m}(s') - \mu_h^m \Biggr)^2,
    \]
    and $\mu_h^m = \sum_{a \in A} \sum_{s' \in S^+} \optimisticpi{m}(a \mid s_h^m) P(s' \mid s_h^m,a) \optimisticctg{m}(s')$.
\end{lemma}

\begin{proof}
    Denote
    \begin{align*}
        X^m 
        & = 
        \sum_{h=1}^{\Hm} \Biggl( \optimisticctg{m}(s_{h+1}^m) - \sum_{a \in A} \sum_{s' \in S} \optimisticpi{m}(a \mid s_h^m) \optimisticP{m}(s' \mid s_h^m,a) \optimisticctg{m}(s') \Biggr) \indgeventi{m}
        \\
        Z_h^m
        & =
        \Biggl( \optimisticctg{m}(s_{h+1}^m) - \sum_{a \in A} \sum_{s' \in S} \optimisticpi{m}(a \mid s_h^m) P(s' \mid s_h^m,a) \optimisticctg{m}(s') \Biggr) \indgeventi{m}.
    \end{align*}
    Think of the interval as an infinite stochastic process, and note that, conditioned on $\trajconcat{m-1}$,
    $
        \bigl(Z_h^m \bigr)_{h=1}^\infty
    $
    is a martingale difference sequence w.r.t $(\traj{h})_{h=1}^\infty$, where $\traj{h}$ is the trajectory of the learner from the beginning of the interval and up to and including time $h$. This holds since, by conditioning on $\trajconcat{m-1}$, $\geventi{m}$ is determined and is independent of the randomness generated during the interval. 
    
    Note that $\Hm$ is a stopping time with respect to $(Z_h^m)_{h=1}^\infty$ which is bounded by $2D / \cmin^2$. 
    Hence by the optional stopping theorem
    $
        \bbE [ \sum_{h=1}^{\Hm} Z_h^m \mid \trajconcat{m-1}] 
        = 
        0
    $,
    which gets us
    \begin{align*}
        \bbE & [X^m \mid \trajconcat{m-1} ]
        =
        \\
        & =
        \bbE \Biggl[ \sum_{h=1}^{\Hm} \Biggl( \optimisticctg{m}(s_{h+1}^m) - \sum_{a \in A} \sum_{s' \in S} \optimisticpi{m}(a \mid s_h^m) \optimisticP{m}(s' \mid s_h^m,a) \optimisticctg{m}(s') \Biggr) \indgeventi{m} \mid \trajconcat{m-1} \Biggr]
        \\
        & =
        \bbE \Biggl[ \sum_{h=1}^{\Hm} Z_h^m \mid \trajconcat{m-1} \Biggr]
        +
        \bbE \Biggl[ \sum_{h=1}^{\Hm} \sum_{a \in A} \sum_{s' \in S} \bigl( P(s' \mid s_h^m,a) - \optimisticP{m}(s' \mid s_h^m,a) \bigr) \optimisticpi{m}(a \mid s_h^m) \optimisticctg{m}(s') \indgeventi{m} \mid \trajconcat{m-1} \Biggr]
        \\
        & =
        \bbE \Biggl[ \sum_{h=1}^{\Hm} \sum_{a \in A} \sum_{s' \in S} \bigl( P(s' \mid s_h^m,a) - \optimisticP{m}(s' \mid s_h^m,a) \bigr) \optimisticpi{m}(a \mid s_h^m) \optimisticctg{m}(s') \indgeventi{m} \mid \trajconcat{m-1} \Biggr].
    \end{align*}
    
    Furthermore, we have
    \begin{align*}
        \bbE \Biggl[ & \sum_{h=1}^{\Hm} \sum_{a \in A} \sum_{s' \in S} \bigl( P(s' \mid s_h^m,a) - \optimisticP{m}(s' \mid s_h^m,a) \bigr) \optimisticpi{m}(a \mid s_h^m) \optimisticctg{m}(s') \indgeventi{m} \mid \trajconcat{m-1} \Biggr] =
        \\
        & =
        \bbE \Biggl[ \sum_{h=1}^{\Hm} \sum_{a \in A} \sum_{s' \in S^+} \bigl( P(s' \mid s_h^m,a) - \optimisticP{m}(s' \mid s_h^m,a) \bigr) \optimisticpi{m}(a \mid s_h^m) \optimisticctg{m}(s') \indgeventi{m} \mid \trajconcat{m-1} \Biggr]
        \\
        & =
        \bbE \Biggl[ \sum_{h=1}^{\Hm} \sum_{s' \in S^+} \bigl( P(s' \mid s_h^m,a_h^m) - \optimisticP{m}(s' \mid s_h^m,a_h^m) \bigr) \optimisticctg{m}(s') \indgeventi{m} \mid \trajconcat{m-1} \Biggr]
        \\
        & =
        \bbE \Biggl[ \sum_{h=1}^{\Hm} \sum_{s' \in S^+} \bigl( P(s' \mid s_h^m,a_h^m) - \optimisticP{m}(s' \mid s_h^m,a_h^m) \bigr)
        \Biggl(\optimisticctg{m}(s') - \mu_h^m \Biggr) \indgeventi{m} \mid \trajconcat{m-1} \Biggr]
        \\
        & \le
        \bbE \Biggl[ 8 \sum_{h=1}^{\Hm} \sum_{s' \in S^+} \sqrt{A_h^m P(s' \mid s_h^m,a_h^m) \Biggl(\optimisticctg{m}(s') - \mu_h^m \Biggr)^2} \indgeventi{m} \mid \trajconcat{m-1} \Biggr]
        \\
        & \qquad +
        \bbE \Biggl[136 \sum_{h=1}^{\Hm} \sum_{s' \in S^+} A_h^m \Biggl| \optimisticctg{m}(s') - \mu_h^m \Biggr| \indgeventi{m} \mid \trajconcat{m-1} \Biggr]
        \\  
        & \le
        \bbE \Biggl[ 16 \sum_{h=1}^{\Hm} \sqrt{ |S| \bbV_h^m A_h^m } \indgeventi{m} + 272 |S| \frac{D}{\cmin} A_h^m \indgeventi{m} \mid \trajconcat{m-1} \Biggr],
    \end{align*}
    where the first equality follows because $\optimisticctg{m}(\ssink)=0$ and the second by the definition of $a_h^m$.
    The third equality follows since  $P(\cdot \mid s_h^m, a_h^m)$ and $\optimisticP{m}(\cdot \mid s_h^m, a_h^m)$ are probability distributions over $S^+$ whence $\mu_h^m$ does not depend on $s'$.
    The first inequality follows from \cref{lem:relbernstein}, and the second inequality from Jensen's inequality, \cref{lem:bern-opt-val-bound}, $|S^+| \le 2 |S|$, and the definition of $\bbV_h^m$.
\end{proof}

The following lemma will help us bound the variance within an interval, and it follows by the fact that known states were visited many times so our estimation of the transition function in these states is relatively accurate.

\begin{lemma}[\cite{cohen2020ssp}, Lemma B.14]
    \label{lem:bern-known-state}
    Let $m$ be an interval and $s$ be a known state. 
    If $\geventi{m}$ holds then for every $a \in A$ and $s' \in S^+$,
    \[
        \bigl| \optimisticP{m} \bigl(s' \mid s, a \bigr)
        -
        P \bigl(s' \mid s, a \bigr) \bigr|
        \le
        \frac{1}{8} \sqrt{\frac{\cmin^2 \cdot P \bigl(s' \mid s, a \bigr)}{|S| D}} 
        + 
        \frac{\cmin^2}{4 |S| D}.
    \]
\end{lemma}

Define $\mu^m(s) = \sum_{a \in A} \sum_{s' \in S^+} \optimisticpi{m}(a \mid s) P(s' \mid s,a) \optimisticctg{m}(s')$ and therefore $\mu_h^m = \mu^m(s_h^m)$.
Similarly, define $\bbV^m(s,a) = \sum_{s' \in S^+} P(s' \mid s,a) \Biggl(\optimisticctg{m}(s') - \mu^m(s) \Biggr)^2$ and therefore $\bbV_h^m = \bbV^m(s_h^m,a_h^m)$.
The next lemma bounds the variance within a single interval.

\begin{lemma}
    \label{lem:bounded-V-h-i}
    For any interval $m$ it holds that
    $
        \bbE \bigl[\sum_{h=1}^{\Hm} \bbV_h^m \indgeventi{m} \mid \trajconcat{m-1} \bigr]
        \le 
        64 D^2 / \cmin^2.
    $
\end{lemma}

\begin{proof}
    Denote
    \[
        Z_h^m 
        = 
        \Biggl( \optimisticctg{m}(s_{h+1}^m) - \sum_{a \in A} \sum_{s' \in S} \optimisticpi{m}(a \mid s_h^m) P(s' \mid s_h^m,a) \optimisticctg{m}(s') \Biggr) \indgeventi{m},
    \]
    and think of the interval as an infinite stochastic process. Note that, conditioned on $\trajconcat{m-1}$,
    $
        \bigl(Z_h^m \bigr)_{h=1}^\infty
    $
    is a martingale difference sequence w.r.t $(\traj{h})_{h=1}^\infty$, where $\traj{h}$ is the trajectory of the learner from the beginning of the interval and up to time $h$ and including. This holds since, by conditioning on $\trajconcat{m-1}$, $\geventi{m}$ is determined and is independent of the randomness generated during the interval. 
    Note that $\Hm$ is a stopping time with respect to $(Z_h^m)_{h=1}^\infty$ which is bounded by $2D / \cmin^2$. 
    Therefore, applying \cref{lem:martingalevariance} obtains
    \begin{equation} 
        \label{eq:timeslotvariance}
        \bbE \Biggl[\sum_{h=1}^{\Hm} \bbV_h^m \ind\{\geventi{m}\} \mid \trajconcat{m-1} \Biggr] 
        =
        \bbE \Biggl[\Biggl(\sum_{h=1}^{\Hm} Z_h^m \ind\{\geventi{m}\} \Biggr)^2  \mid \trajconcat{m-1} \Biggr].
    \end{equation}
    We now proceed by bounding $|\sum_{h=1}^{\Hm} Z_h^m |$ when $\geventi{m}$ occurs.
    Therefore,
    \begin{align}
        \nonumber 
        \Biggl| & \sum_{h=1}^{\Hm} Z_h^m \Biggr|
        =
        \Biggl|\sum_{h=1}^{\Hm} \optimisticctg{m}(s_{h+1}^m) - \sum_{a \in A} \sum_{s' \in S} \optimisticpi{m}(a \mid s_h^m) P(s' \mid s_h^m,a) \optimisticctg{m}(s') \Biggr| 
        \\
        \label{eq:bern-var-bound-telescope}
        & \le
        \Biggl|\sum_{h=1}^{\Hm} \optimisticctg{m}(s_{h+1}^m) - \optimisticctg{m}(s_{h}^m) \Biggr| 
        \\
        \label{eq:bern-var-bound-bellman}
        & \qquad + 
        \Biggl|\sum_{h=1}^{\Hm} \optimisticctg{m}(s_{h}^m) - \sum_{a \in A} \sum_{s' \in S} \optimisticpi{m}(a \mid s_h^m) \optimisticP{m}(s' \mid s_h^m,a) \optimisticctg{m}(s') \Biggr| 
        \\
        \label{eq:bern-var-bound-err}
        & \qquad +
        \Biggl| \sum_{h=1}^{\Hm} \sum_{a \in A} \sum_{s' \in S^+} \optimisticpi{m}(a \mid s_h^m) \Bigl(  \optimisticP{m}(s' \mid s_h^m,a) - P(s' \mid s_h^m,a) \Bigr) \Bigl(\optimisticctg{m}(s') - \mu_h^m \Bigr) \Biggr|,
    \end{align}
    where \cref{eq:bern-var-bound-err} is given as $P(\cdot \mid s_h^m, a)$ and $\optimisticP{m}(\cdot \mid s_h^m, a)$ are probability distributions over $S^+$, $\mu_h^m$ is constant w.r.t $s'$, and $\optimisticctg{m}(\ssink) = 0$.
    
    We now bound each of the three terms above individually.
    \cref{eq:bern-var-bound-telescope} is a telescopic sum that is at most $D/\cmin$ on $\geventi{m}$ (\cref{lem:bern-opt-val-bound}).
    For \cref{eq:bern-var-bound-bellman}, we use the Bellman equations for $\optimisticpi{m}$ w.r.t $\optimisticP{m}$ (\cref{lem:bellman-optimistic}) thus it is at most $2D/\cmin$
    (see proof of \cref{lem:optimistic-val-func-diff-to-expected}).
    For \cref{eq:bern-var-bound-err}, recall that all states at times $h=1,\ldots,\Hm$ are known by definition of $\Hm$. Hence by \cref{lem:bern-known-state}, 
    \begin{align*}
        \Biggl| \sum_{s' \in S^+} \Bigl(  P(s' \mid s_h^m,a) - \optimisticP{m}(s' \mid s_h^m,a) \Bigr) \Bigl(\optimisticctg{m}(s') - \mu_h^m \Bigr) \Biggr| 
        & \le
        \frac{1}{8} \sum_{s' \in S^+} \sqrt{\frac{\cmin^2 P(s' \mid s_h^m,a) \Bigl(\optimisticctg{m}(s') - \mu_h^m \Bigr)^2}{|S| D}}
        \\
        & \qquad +
        \sum_{s' \in S^+} \frac{\cmin^2}{4 |S| D} \underbrace{\Bigl|\optimisticctg{m}(s') - \mu_h^m \Bigr|}_{\le D/\cmin}
        \\
        & \le
        \frac{1}{4} \sqrt{\frac{\cmin^2 \bbV^m(s_h^m,a)}{D}}
        + \frac{\cmin}{2},
    \end{align*}
    where the last inequality follows from Jensen's inequality and because $|S^+| \le 2|S|$.
    Therefore,
    \begin{align*}
        \Biggl| \sum_{a \in A} & \sum_{s' \in S^+} \optimisticpi{m}(a \mid s_h^m) \Bigl(  P(s' \mid s_h^m,a) - \optimisticP{m}(s' \mid s_h^m,a) \Bigr) \Bigl(\optimisticctg{m}(s') - \mu_h^m \Bigr) \Biggr| 
        \le
        \\
        & \le
        \sum_{a \in A} \optimisticpi{m}(a \mid s_h^m) \biggl( \frac{1}{4} \sqrt{\frac{\cmin^2 \bbV^m(s_h^m,a)}{D}}
        + \frac{\cmin}{2} \biggr)
        \\
        & \le
        \frac{1}{4} \sqrt{\frac{\cmin^2 \sum_{a \in A} \optimisticpi{m}(a \mid s_h^m) \bbV^m(s_h^m,a)}{D}}
        + \frac{\cmin}{2},
    \end{align*}
    where the last inequality follows again from Jensen's inequality.
    We use Jensen's inequality one last time to obtain
    \begin{align*}
        \sum_{h=1}^{\Hm} & \frac{1}{4} \sqrt{\frac{\cmin^2 \sum_{a \in A} \optimisticpi{m}(a \mid s_h^m) \bbV^m(s_h^m,a)}{D}}
        +
        \sum_{h=1}^{\Hm} \frac{\cmin}{2}
        \le
        \\
        & \le \frac{1}{4} \sqrt{\Hm \sum_{h=1}^{\Hm} \frac{\cmin^2 \sum_{a \in A} \optimisticpi{m}(a \mid s_h^m) \bbV^m(s_h^m,a)}{D}}
         + \frac{\cmin \Hm}{2}
         \\
         & \le
         \frac{1}{2} \sqrt{\sum_{h=1}^{\Hm} \sum_{a \in A} \optimisticpi{m}(a \mid s_h^m) \bbV^m(s_h^m,a)}
         + \frac{D}{\cmin},
    \end{align*}
    where we used the fact that $\Hm \le 2D/\cmin^2$.
    
    Plugging these bounds back into \cref{eq:timeslotvariance} gets us
    \begin{align*}
        \bbE \Biggl[\sum_{h=1}^{\Hm} \bbV_h^m \indgeventi{m} \biggm| \trajconcat{m-1} \Biggr] \nonumber
        & \le
        \bbE \Biggl[
        \Biggl(
        \frac{4D}{\cmin}
        + 
        \frac{1}{2} \sqrt{\sum_{h=1}^{\Hm} \sum_{a \in A} \optimisticpi{m}(a | s_h^m) \bbV^m(s_h^m,a) \indgeventi{m}}
        \Biggr)^2 \biggm| \trajconcat{m-1} \Biggr] 
        \\
        &\le
        \frac{32 D^2}{\cmin^2}
        +
        \frac{1}{2} \bbE \Biggl[\sum_{h=1}^{\Hm} \sum_{a \in A} \optimisticpi{m}(a | s_h^m) \bbV^m(s_h^m,a) \indgeventi{m} \biggm| \trajconcat{m-1} \Biggr]
        \\
        & =
        \frac{32 D^2}{\cmin^2}
        +
        \frac{1}{2} \bbE \Biggl[\sum_{h=1}^{\Hm} \bbV_h^m \indgeventi{m} \biggm| \trajconcat{m-1} \Biggr], 
    \end{align*}
    where the second inequality is by the elementary inequality $(a+b)^2 \le 2(a^2 + b^2)$, and the last equality is by definition of $a_h^m$ and $\bbV_h^m$.
    Rearranging gets us $\bbE \bigl[\sum_{h=2}^{\Hm} \bbV_h^m \indgeventi{m} \mid \trajconcat{m-1} \bigr] \le 64 D^2/\cmin^2$, and the lemma follows.
\end{proof}

\begin{lemma}
    \label{lem:sum-bern-bounds-cont}
    With probability at least $1 - \delta / 6$,  the following holds for all $M = 1,2,\ldots$ simultaneously.
    \begin{align*}
        \sum_{m=1}^\numintervals \bbE [\tregret{m}^2 \mid \trajconcat{m-1}] 
        & \le
        573 \frac{D|S|}{\cmin} \sqrt{\numintervals |A| \log^2 \frac{\totaltime |S| |A|}{\delta}}
        + 
        5440 \frac{D}{\cmin} |S|^2 |A| \log^2 \frac{\totaltime |S| |A|}{\delta}.
    \end{align*}
\end{lemma}

\begin{proof}
    From \cref{lem:sum-bern-bounds} we have that 
    \begin{align*}
        \bbE [ \tregret{m}^2 \mid \trajconcat{m-1} ] 
        & \le 
        16 \bbE \Biggl[ 
        \sum_{h=1}^{\Hm} \sqrt{|S| \bbV_h^m A_h^m} \indgeventi{m} 
        \biggm| \trajconcat{m-1} \Biggr]
        + 
        272 \bbE \Biggl[ \sum_{h=1}^{\Hm} \frac{D}{\cmin} |S| A_h^m \indgeventi{m} 
        \biggm| \trajconcat{m-1} \Biggr],
    \end{align*}
    
    Moreover, by applying the Cauchy-Schwartz inequality twice, we get that
    \begin{align*}
        \bbE \Biggl[ \sum_{h=1}^{\Hm}  \sqrt{\bbV_h^m A_h^m} & \indgeventi{m} \biggm| \trajconcat{m-1} \Biggr]
        \le 
        \bbE \Biggl[ \sqrt{\sum_{h=1}^{\Hm} \bbV_h^m \indgeventi{m}} \cdot \sqrt{\sum_{h=1}^{\Hm} A_h^m \indgeventi{m}} \biggm| \trajconcat{m-1} \Biggr] 
        \\
        &\le
        \sqrt{\bbE \Biggl[ 
        \sum_{h=1}^{\Hm} A_h^m \indgeventi{m} \biggm| \trajconcat{m-1} \Biggr]} \cdot \sqrt{\bbE \Biggl[ 
        \sum_{h=1}^{\Hm} \bbV_h^m \indgeventi{m} \biggm| \trajconcat{m-1} \Biggr]} 
        \\
        &\le
        \frac{8D}{\cmin}
        \sqrt{\bbE \Biggl[ 
        \sum_{h=1}^{\Hm} A_h^m \indgeventi{m} \biggm| \trajconcat{m-1} \Biggr]},
    \end{align*}
    where the last inequality is by \cref{lem:bounded-V-h-i}.
    We sum over all intervals to obtain
    \begin{align*}
        \sum_{m=1}^\numintervals \bbE [\tregret{m}^2 \mid \trajconcat{m-1}] 
        & \le
        \frac{128 D}{\cmin} \sum_{m=1}^\numintervals \sqrt{|S| \Biggl[ 
        \sum_{h=1}^{\Hm} A_h^m \indgeventi{m} \biggm| \trajconcat{m-1} \Biggr]}
        +
        \frac{272 D |S|}{\cmin} \sum_{m=1}^\numintervals \Biggl[ 
        \sum_{h=1}^{\Hm} A_h^m \indgeventi{m} \biggm| \trajconcat{m-1} \Biggr]
        \\
        & \le
        \frac{128 D}{\cmin} \sqrt{\numintervals |S| \sum_{m=1}^\numintervals  \Biggl[ 
        \sum_{h=1}^{\Hm} A_h^m \indgeventi{m} \biggm| \trajconcat{m-1} \Biggr]}
        +
        \frac{272 D |S|}{\cmin} \sum_{m=1}^\numintervals \Biggl[ 
        \sum_{h=1}^{\Hm} A_h^m \indgeventi{m} \biggm| \trajconcat{m-1} \Biggr],
    \end{align*}
    where the last inequality follows from Jensen's inequality.
    We finish the proof using \cref{lem:bern-sum-visits} below.
\end{proof}

\begin{lemma}
    \label{lem:bern-sum-visits}
    With probability at least $1 - \delta / 6$, the following holds for $M=1,2,\ldots$ simultaneously.
    \[
        \sum_{m=1}^\numintervals \bbE \Biggl[ \sum_{h=1}^{\Hm} A_h^m \indgeventi{m} \mid \trajconcat{m-1} \Biggr]
        \le
        20 |S| |A| \log^2 \frac{\totaltime |S| |A|}{\delta}.
    \]
\end{lemma}

\begin{proof}
    Define the infinite sequence of random variables: $X^m = \sum_{h=1}^{\Hm} A_h^m \indgeventi{m}$ for which $|X^m| \le 2$ due to \cref{lem:sum-inverse-visit-count-in-interval-bound} below.
    We apply \cref{eq:anytime-bern-3} of \cref{lem:martingalte-multiplicative-bound} to obtain with probability at least $1 - \delta / 6$, for all $M =1,2,\ldots$ simultaneously
    \begin{align*}
        \sum_{m=1}^\numintervals \bbE \bigl[ X^m \mid \trajconcat{m-1} \bigr]
        \le
        2 \sum_{m=1}^\numintervals X^m
        +
        8 \log \frac{12 M}{\delta}.
    \end{align*}
    Now, we bound the sum over $X^m$ by rewriting it as a sum over epochs (since the confidence sets update only in the beginning of a new epoch):
    \[
        \sum_{m=1}^\numintervals X^m
        \le
        \sum_{m=1}^\numintervals \sum_{h=1}^{\Hm} \frac{\log(|S| |A| N_+^{e(m)}(s_{h}^m,a_{h}^m) / \delta)}{N_+^{e(m)}(s_{h}^m,a_{h}^m)}
        \le
        \log \frac{|S| |A| \totaltime}{\delta} \sum_{s \in S} \sum_{a \in A} \sum_{e=1}^E \frac{n^e(s,a)}{N_+^e(s,a)},
    \]
    where $n^e(s,a)$ is the number of visits to $(s,a)$ during epoch $e$. 
    From \cref{lem:sum-inverse-visit-count-bound} below we have that for every $(s,a) \in S \times A$,
    \[
        \sum_{e=1}^{E} \frac{n^e(s,a)}{N_+^e(s,a)}
        \le
        2 \log N_{E+1} (s,a)
        \le
        2 \log T.
    \]
    We now plugin the resulting bound for $\sum_{m=1}^M X^m$ and simplify the acquired expression by using $M \le T$.
\end{proof}

\begin{lemma} 
    \label{lem:sum-inverse-visit-count-in-interval-bound}
    For any interval $m$,
    $
        | \sum_{h=1}^{\Hm} A_h^m | \le 2.
    $
\end{lemma}

\begin{proof}
    Note that all states during the interval are known.
    Hence, $N_+^{e(m)}(s_{h}^m,a_{h}^m) \ge \alpha \cdot \numvisitsuntilknownbern$. Therefore, since $\log(x)/x$ is decreasing and since $|A| \ge 2$ (otherwise the learner has no choices),
    \begin{align*}
        \sum_{h=1}^{\Hm} A^m_h
        =
        \sum_{h=1}^{\Hm} \frac{\log(|S| |A| N_+^{e(m)}(s_{h}^m,a_{h}^m) / \delta)}{N_+^{e(m)}(s_{h}^m,a_{h}^m)}
        \le
        \frac{\cmin^2 \Hm}{D} 
        \le 2.
    \end{align*}
\end{proof}

\begin{lemma}[\cite{cohen2020ssp}, Lemma B.18]
    \label{lem:sum-inverse-visit-count-bound}
    For any sequence of integers $z_1,\dots,z_n$ with $0 \leq z_k \leq Z_{k-1} := \max \{ 1 , \sum_{i=1}^{k-1} z_i \}$ and $Z_0 = 1$, it holds that
    \[
        \sum_{k=1}^n \frac{z_k}{Z_{k-1}} \leq 2 \log Z_n.
    \]
\end{lemma}

\subsection{Proof of \texorpdfstring{\cref{thm:reg-bound-unknown-P}}{}}

\begin{theorem}[Restatement of \cref{thm:reg-bound-unknown-P}]
    Under \cref{ass:c-min}, running SSP-O-REPS3 with known SSP-diameter $D$ and $\eta = \sqrt{\frac{6 \log (D |S| |A| / \cmin)}{K}}$ ensures that, with probability at least $1 - \delta$,
    \[
        \regret
        \le
        O \biggl( \frac{D |S|}{\cmin} \sqrt{|A| K} \log \frac{K D |S| |A|}{\delta \cmin} + \frac{D^2 |S|^2 |A|}{\cmin^2} \log^2 \frac{K D |S| |A|}{\delta \cmin} \biggr)
        =
        \tO \biggl( \frac{D |S|}{\cmin} \sqrt{|A| K} \biggr),
    \]
    where the last equality holds for $K \ge D^2 |S|^2 |A| / \cmin^2$.
\end{theorem}

\begin{proof}[Proof of \cref{thm:reg-bound-unknown-P}]
    With probability at least $1 - \delta$, via a union bound, we have that \cref{lem:conf-set-hold,lem:sum-bern-bounds-cont,lem:optimistic-val-func-diff-to-expected} hold and the following holds by Azuma inequality for every $T=1,2,\dots$ simultaneously,
    \begin{align}
        \label{eq:reg-tilde-to-actual}
        \sum_{m=1}^M \sum_{h=1}^{\Hm} c_m(s_h^m,a_h^m)
        \le
        \sum_{m=1}^M \sum_{h=1}^{\Hm} \sum_{a \in A} \optimisticpi{m}(a \mid s_h^m) c_m(s_h^m,a)
        +
        4 \sqrt{T \log \frac{T}{\delta}}.
    \end{align}
    We start by bounding $\tregret{M}$ and in the end we explain how this yields a bound on $\regret$.
    
    Plugging in the bounds of \cref{lem:bern-first-reg-term-bound,lem:optimistic-val-func-diff-to-expected,lem:sum-bern-bounds-cont} into \cref{lem:bern-reg-decomp}, we have that for any number of intervals $\numintervals$: 
    \[
        \wt C_M 
        \le
        \sum_{k=1}^K \optimisticctg{\pi_k}_k(\sinit) \indgeventi{m(k)}
        + O \biggl(
        \frac{D|S|}{\cmin} \sqrt{\numintervals |A|} \log \frac{\totaltime |S| |A|}{\delta}
        + 
        \frac{D^2 |S|^2 |A|}{\cmin^2} \log^2 \frac{\totaltime |S| |A|}{\delta} \biggr).
    \]
    We now plug in the bound on $M$ from \cref{obs:cost-bounds-intervals} into the bound above.
    After simplifying this gets us
    \begin{align*}
        \wt C_M 
        & \le
        \sum_{k=1}^K \optimisticctg{\pi_k}_k(\sinit) \indgeventi{m(k)}
        + O \biggl( \sqrt{\frac{D^2 |S|^2 |A|}{\cmin^2} K \log^2 \frac{T D |S| |A|}{\delta \cmin}} 
        \\
        & \quad +
        \sqrt{\frac{D^4 |S|^4 |A|^2}{\cmin^4} \log^4 \frac{T D |S| |A|}{\delta \cmin}} 
        +
        \sqrt{\frac{D |S|^2 |A|}{\cmin} \wt C_M \log^2 \frac{T D |S| |A|}{\delta \cmin}} \biggr).
    \end{align*}
    From which, by solving for $\wt C_M$ (using that $x \le a \sqrt{x} + b$ implies $x \le (a + \sqrt{b})^2$ for $a \ge 0$ and $b \ge 0$), and simplifying the resulting expression by applying $\optimisticctg{\pi_k}_k(\sinit) \le D/\cmin$ and our assumptions that $K \ge |S|^2 |A|$, $|A| \ge 2$, we get that
    \begin{align}
        \label{eq:last-reg-bound-with-T}
        \wt C_M 
        & \le
        \sum_{k=1}^K \optimisticctg{\pi_k}_k(\sinit) \indgeventi{m(k)}
         + 
        O \biggl( \frac{D |S|}{\cmin} \sqrt{|A| K} \log \frac{T D |S| |A|}{\delta \cmin}
        + \frac{D^2 |S|^2 |A|}{\cmin^2} \log^2 \frac{T D |S| |A|}{\delta \cmin} \biggr).
    \end{align}
    Note that in particular, by simplifying the bound above, we obtain a polynomial bound on the total cost:
    $
        \wt C_M
        =
        O \Bigl(\sqrt{D^4 |S|^4 |A|^2 K \totaltime / \cmin^4 \delta} \Bigr).
    $
    Next we combine this with the fact, stated in
    \cref{obs:cost-bounds-intervals} that $T\le \wt C_M/ \cmin$. Isolating $T$ gets
    $
        \totaltime
        =
        O \Bigl(\tfrac{D^4 |S|^4 |A|^2 K}{\cmin^4 \delta} \Bigr),
    $
    and plugging this bound back into \cref{eq:last-reg-bound-with-T} and simplifying gets us
    \begin{align}
        \wt C_M 
        & \le
        \sum_{k=1}^K \optimisticctg{\pi_k}_k(\sinit) \indgeventi{m(k)}
        + O \biggl( \frac{D |S|}{\cmin} \sqrt{|A| K} \log \frac{K D |S| |A|}{\delta \cmin}
        + \frac{D^2 |S|^2 |A|}{\cmin^2} \log^2 \frac{K D |S| |A|}{\delta \cmin} \biggr).
    \end{align}
    Recall that
    \[
        \sum_{k=1}^K \optimisticctg{\pi_k}_k(\sinit) - \ctg{\pi^\star}_k(\sinit) 
        = 
        \sum_{k=1}^K \langle q_k - q^{P,\pi^\star} , c_k \rangle,
    \]
    and thus applying OMD analysis (see \cref{sec:omd-unknown-P-proofs}) we obtain
    \[
        \tregret{M}
        \le
        O \biggl( \frac{D |S|}{\cmin} \sqrt{|A| K} \log \frac{K D |S| |A|}{\delta \cmin}
        + \frac{D^2 |S|^2 |A|}{\cmin^2} \log^2 \frac{K D |S| |A|}{\delta \cmin} \biggr).
    \]
    Now, as $\geventi{m}$ hold for all intervals, we use \cref{eq:reg-tilde-to-actual} to bound the actual regret (together with $T \le \wt C_M/\cmin$) for any number of intervals $M$, with the bound we have for $\tregret{M}$.

    We note that the bound above holds for any number of intervals $M$ as long as $K$ episodes do not elapse. As the instantaneous costs in the model are positive, this means that the learner must eventually finish the $K$ episodes from which we derive the bound for $\regret$ claimed by the theorem.
\end{proof}

\subsection{OMD analysis}
\label{sec:omd-unknown-P-proofs}

This analysis follows the lines of \cref{sec:omd-ssp-proofs}, but it is adjusted to extended occupancy measures.

\begin{lemma}
    \label{lem:ext-bound-R-q-pos}
    Let $\tau \ge 1$.
    For every $q \in \wtboundedocset{\tau}{m}$ it holds that $R(q) \le \tau \log \tau$.
\end{lemma}

\begin{proof}
    \begin{align*}
        R(q)
        & =
        \sum_{s \in S} \sum_{a \in A} \sum_{s' \in S^+} q(s,a,s') \log q(s,a,s') - \sum_{s \in S} \sum_{a \in A} \sum_{s' \in S^+} q(s,a,s')
        \\
        & \le
        \sum_{s \in S} \sum_{a \in A} \sum_{s' \in S^+} q(s,a,s') \log q(s,a,s')
        \\
        & =
        \sum_{s \in S} \sum_{a \in A} \sum_{s' \in S^+} q(s,a,s') \log \frac{q(s,a,s')}{\tau} + \sum_{s \in S} \sum_{a \in A} \sum_{s' \in S^+} q(s,a,s') \log \tau
        \\
        & \le
        \sum_{s \in S} \sum_{a \in A} \sum_{s' \in S^+} q(s,a,s') \log \tau
        \le
        \tau \log \tau,
    \end{align*}
    where the first two inequalities follow from non-positivity, and the last one from the definition of $\wtboundedocset{\tau}{m}$.
\end{proof}

\begin{lemma}
    \label{lem:ext-bound-R-q-neg}
    Let $\tau \ge 1$.
    For every $q \in \wtboundedocset{\tau}{m}$ it holds that $-R(q) \le \tau (1 + \log (|S|^2 |A|))$.
\end{lemma}

\begin{proof}
    Similarly to \cref{lem:bound-R-q-pos} we have that
    \begin{align*}
        - R(q)
        & =
        - \sum_{s \in S} \sum_{a \in A} \sum_{s' \in S^+} q(s,a,s') \log \frac{q(s,a,s')}{\tau}
        +
        \sum_{s \in S} \sum_{a \in A} \sum_{s' \in S^+} q(s,a,s')
        \\
        & \qquad -
        \sum_{s \in S} \sum_{a \in A} \sum_{s' \in S^+} q(s,a,s') \log \tau
        \\
        & \le
        - \tau \sum_{s \in S} \sum_{a \in A} \sum_{s' \in S^+} \frac{q(s,a,s')}{\tau} \log \frac{q(s,a,s')}{\tau} + \tau
        \le
        \tau \log (|S|^2 |A|) + \tau,
    \end{align*}
    where the first inequality follows because the last term is non-positive and from the definition of $\wtboundedocset{\tau}{m}$, and the last inequality follows from properties of Shannon's entropy.
\end{proof}

\begin{lemma}
    \label{lem:omd-reg-bound-unknown-P}
    If $\geventi{m}$ holds for all intervals $m$, then
    \[
        \sum_{k=1}^K \langle q_k - q^{P,\pi^\star} , c_k \rangle
        \le
        \frac{2D}{\cmin} \sqrt{6 K \log \frac{D |S| |A|}{\cmin}}.
    \]
\end{lemma}

\begin{proof}
    We start with a fundamental inequality of OMD (see, e.g., \cite{rosenberg2019full}) that holds for every $q \in \wtboundedocset{D/\cmin}{m}$ for every $m$ (since $\geventi{m}$ holds it also holds for $q^{P,\pi^\star}$),
    \begin{align}
        \label{eq:ext-standard-omd}
        \sum_{k=1}^K \langle q_k - q^{P,\pi^\star} , c_k \rangle
        \le
        \sum_{k=1}^K \langle q_k - q'_{k+1} , c_k \rangle
        +
        \frac{\KL{q^{P,\pi^\star}}{q_1}}{\eta}.
    \end{align}
    For the first term we use the exact form of $q'_{k+1}$ and the inequality $e^x \ge 1+x$ to obtain
    \[
        q'_{k+1}(s,a,s')
        =
        q_k(s,a,s') e^{- \eta c_k(s,a)}
        \ge
        q_k(s,a,s') - \eta q_k(s,a,s') c_k(s,a).
    \]
    We substitute this back and obtain
    \begin{align}
        \nonumber
        \sum_{k=1}^K \langle q_k - q'_{k+1} , c_k \rangle
        & \le
        \eta \sum_{k=1}^K \sum_{s \in S} \sum_{a \in A} \sum_{s' \in S^+} q_k(s,a,s') c_k(s,a)^2
        \le
        \eta \sum_{k=1}^K \sum_{s \in S} \sum_{a \in A} \sum_{s' \in S^+} q_k(s,a,s')
        \\
        \label{eq:ext-omd-term-1}
        & =
        \eta \sum_{k=1}^K \optimisticpolicytime{\pi_k}_k(\sinit)
        \le
        \eta K \frac{D}{\cmin},
    \end{align}
    where the last inequality follows from the definition of $\wtboundedocset{D/\cmin}{m(k)}$.
    
    Next we use \cref{lem:ext-bound-R-q-pos,lem:ext-bound-R-q-neg} to bound the second term of \cref{eq:ext-standard-omd}.
    Recall that $q_1$ minimizes $R$ in $\wtboundedocset{D/\cmin}{1}$, this implies that $\langle \nabla R(q_1), q^{P,\pi^\star} - q_1 \rangle \ge 0$ because otherwise we could decrease $R$ by taking small step in the direction $q^{P,\pi^\star} - q_1$.
    Thus we obtain
    \begin{align}
        \nonumber
        \KL{q^{P,\pi^\star}}{q_1}
        & =
        R(q^{P,\pi^\star}) - R(q_1) - \langle \nabla R(q_1), q^{P,\pi^\star} - q_1 \rangle
        \le
        R(q^{P,\pi^\star}) - R(q_1)
        \\
        \label{eq:ext-omd-term-2}
        & \le 
        \frac{D}{\cmin} \log \frac{D}{\cmin}
        +
        \frac{D}{\cmin} (1 + \log (|S|^2 |A|))
        \le
        \frac{6D}{\cmin} \log \frac{D |S| |A|}{\cmin}.
    \end{align}
    
    By substituting \cref{eq:ext-omd-term-1,eq:ext-omd-term-2} into \cref{eq:ext-standard-omd} and choosing $\eta = \sqrt{\frac{6 \log \frac{D |S| |A|}{\cmin}}{K}}$, we obtain,
    \begin{align*}
        \sum_{k=1}^K \langle q_k - q^{P,\pi^\star} , c_k \rangle
        \le
        \eta K \frac{D}{\cmin} + \frac{6D}{\cmin  \eta} \log \frac{D |S| |A|}{\cmin}
        \le
        \frac{2D}{\cmin} \sqrt{6 K \log \frac{D |S| |A|}{\cmin}}.
    \end{align*}
\end{proof}

\newpage
\section{Estimating the SSP-diameter}
\label{sec:estimate-diameter-proofs}

When $D$ is given, we use it to get the upper bound $D/\cmin$ on the expected time of the best policy in hindsight $\policytime{\pi^\star}(\sinit)$.
The reason that $\policytime{\pi^\star}(\sinit) \le D/\cmin$ is that $D$ is an upper bound on the expected time of the fast policy, i.e., $\policytime{\fastpolicy}(\sinit) \le D$ (see \cref{lem:best-in-bounded-oc-set}).
Thus, we want to compute $\wt D(\sinit)$ to be an upper bound on $\policytime{\fastpolicy}(\sinit)$.

We would like to use the first $L$ episodes in order to estimate an upper bound $\wt D(\sinit)$ on the expected time of the fast policy, and then we can run SSP-O-REPS3 and obtain the same regret bound as in \cref{thm:reg-bound-unknown-P} but with $\wt D(\sinit)$ replacing $D$.

Notice that $\fastpolicy$ is the optimal policy w.r.t the constant cost function $c(s,a) = 1$, and its expected cost is $\policytime{\fastpolicy}(\sinit)$.
Thus, we run the SSP regret minimization algorithm of \cite{cohen2020ssp} with the cost function $c(s,a)=1$ for $L$ episodes.
Then, we set $\wt D(\sinit)$ to be the average cost per episode times 10, that is,
\[
    \wt D(\sinit)
    =
    \frac{10}{L} \sum_{k=1}^L \sum_{i=1}^{\Ik} c(s_i^k,a_i^k)
    =
    \frac{10}{L} \sum_{k=1}^L \Ik .
\]

We start by showing that $\wt D(\sinit)$ is indeed an upper bound on $\policytime{\fastpolicy}(\sinit)$, given $L$ is large enough.

\begin{lemma}
    \label{lem:wtD-upp-bound-on-D}
    If $L \ge \frac{2400 D^2}{\policytime{\fastpolicy}(\sinit)^2} \log^3 \frac{4K}{\delta}$ then, with probability at least $1 - \delta$, $\policytime{\fastpolicy}(\sinit) \le \wt D(\sinit)$.
\end{lemma}

\begin{proof}
    Notice that playing $\fastpolicy$ during the first $L$ episodes will result in smaller total cost then running the regret minimization algorithm.
    Thus, it suffices to prove the Lemma as if we are playing the fast policy.
    Define
    \[
        X_k 
        =  
        \sum_{i=1}^{\Ik} c(s^k_i,a^k_i)
        -
        \bbE \Bigl[ \sum_{i=1}^{\Ik} c(s^k_i,a^k_i) \mid P,\fastpolicy,s_1^k=\sinit \Bigr]
        =
        \sum_{i=1}^{\Ik} c(s^k_i,a^k_i)
        -
        \policytime{\fastpolicy}(\sinit).
    \]
    This is a martingale difference sequence, and in order to use \cref{thm:unbounded-azuma} we need to show that $\Pr[|X_k| > m] \le 2e^{-\frac{m}{4 D}}$ for every $k=1,2,\dots$ and $m \ge 0$.
    This follows immediately from \cref{lem:high-prob-time} since the total cost is equal to the total time for the cost function $c(s,a)=1$.
    
    By \cref{thm:unbounded-azuma}, $\left| \sum_{k=1}^L X_k \right| \le 44 D \sqrt{L \log^3 \frac{4L}{\delta}}$ with probability $1 - \delta$.
    Therefore we have,
    \[
        \sum_{k=1}^L \sum_{i=1}^{\Ik} c(s_i^k,a_i^k)
        \ge
        L \policytime{\fastpolicy}(\sinit) - 44 D \sqrt{L \log^3 \frac{4L}{\delta}},
    \]
    and thus,
    \begin{align}
        \label{eq:wtD-upper-bound}
        \frac{\wt D(\sinit)}{10} = \frac{1}{L} \sum_{k=1}^L \sum_{i=1}^{\Ik} c(s_i^k,a_i^k)
        \ge
        \policytime{\fastpolicy}(\sinit) - 44 D \sqrt{\frac{\log^3 \frac{4L}{\delta}}{L}}.
    \end{align}
    Since $L \ge \frac{2400 D^2}{\policytime{\fastpolicy}(\sinit)^2} \log^3 \frac{4K}{\delta}$, we have that $44 D \sqrt{\frac{\log^3 \frac{4L}{\delta}}{L}} \le \frac{9}{10} \policytime{\fastpolicy}(\sinit)$ and therefore we obtain from \cref{eq:wtD-upper-bound} that $\policytime{\fastpolicy}(\sinit) \le \wt D(\sinit)$.
\end{proof}

Next, we show that $\wt D(\sinit)$ is a good estimation of $\policytime{\fastpolicy}(\sinit)$, given $L$ is large enough.

\begin{lemma}
    \label{lem:wt-D-not-big-1}
    If $L \ge |S|^2 |A| \sqrt{D} \log^2 \frac{K D |S| |A|}{\delta}$ then, with probability at least $1 - \delta$, $\wt D(\sinit) \le O(D)$.
\end{lemma}

\begin{proof}
    By the regret bound of the SSP regret minimization algorithm we have, with probability at least $1 - \delta$,
    \[
        \frac{1}{L} \sum_{k=1}^L \sum_{i=1}^{\Ik} c(s_i^k,a_i^k)
        -
        \policytime{\fastpolicy}(\sinit)
        \le
        O \left( \frac{D |S| \sqrt{|A|} \log \frac{L D |S| |A|}{\delta}}{\sqrt{L}} + \frac{D^{3/2} |S|^2 |A| \log^2 \frac{L D |S| |A|}{\delta}}{L} \right).
    \]
    Since $\policytime{\fastpolicy}(\sinit) \le D$ we obtain
    \[
        \wt D(\sinit)
        \le
        O \left( 
        D + \frac{D |S| \sqrt{|A|} \log \frac{L D |S| |A|}{\delta}}{\sqrt{L}} + \frac{D^{3/2} |S|^2 |A| \log^2 \frac{L D |S| |A|}{\delta}}{L} \right) 
        \le 
        O(D).
    \]
    where the last inequality follows because $L \ge |S|^2 |A| \sqrt{D} \log^2 \frac{KD |S| |A|}{\delta}$.
\end{proof}

The second place in which the SSP-O-REPS3 algorithm uses $D$ is to determine when to switch to the optimistic fast policy.
The switch happens when we reach a state with expected time larger than $D/\cmin$.
A careful look at the analysis (especially \cref{lem:bern-first-reg-term-bound}) shows that we actually need to switch in state $s$ if the expected time is larger than $\policytime{\fastpolicy}(s)/\cmin$.
Thus, we need to estimate an upper bound $\wt D(s)$ on $\policytime{\fastpolicy}(s)$ which is done exactly as we estimated $\policytime{\fastpolicy}(\sinit)$, i.e., in the first $L$ visits to $s$ we switch to the optimistic fast policy and then we can estimate  $\policytime{\fastpolicy}(s)$ (by taking the average time to the goal times 10).
This just means that now the threshold for a state to become known is $L$ instead of $\numvisitsuntilknownbern$.
Assuming $L$ is large enough we get a good enough estimate, similarly to what we just proved for $\policytime{\fastpolicy}(\sinit)$.

To summarize, the algorithm proceeds as follows.
We start by running the regret minimization algorithm of \cite{cohen2020ssp} with constant cost of $1$ for $L$ episodes and use it to estimate an upper bound on $\policytime{\fastpolicy}(\sinit)$.
Then, we run SSP-O-REPS3 (setting $\eta$ as a function of our estimate instead of $D$) with a known state threshold of $L$.
When a state $s$ becomes known we compute an upper bound on $\policytime{\fastpolicy}(s)$ and in the next episodes we make the switch in this state using this estimate and not when the expected time is larger than $D/\cmin$.
The following theorem shows that we can set $L \approx \sqrt{K}$, and this leads to the same regret bound (as if we knew $D$ in advance) assuming $K$ is large enough.
Otherwise, the regret is just bounded by some constant that does not depend on $K$.

\begin{theorem}
    \label{thm:reg-bound-unknown-P-unknown-D}
    Under \cref{ass:c-min}, running SSP-O-REPS3 with $\eta = \sqrt{\frac{3 \log (\wt D(\sinit) |S| |A| / \cmin)}{K}}$ and \newline  $L = 2400 \max \{ |S|^2 |A| \log^2 \frac{K |S| |A|}{\delta \cmin} , \frac{\sqrt{K}}{\cmin \sqrt{|A|}} \log \frac{K |S| |A|}{\delta \cmin} \}$ ensures that, with probability at least $1 - \delta$,
    \[
        \regret
        \le
        O \left( \frac{D|S|}{\cmin} \sqrt{|A| K} \log \frac{K D |S| |A|}{\delta \cmin} + \frac{D^2 |S|^2 |A|}{\cmin^2} \log^3 \frac{K D |S| |A|}{\delta \cmin} \right),
    \]
    for
    $
        K 
        \ge 
        \max \biggl\{ \cmin^2 D |S|^4 |A|^3 \log^2 \frac{D |S| |A|}{\delta \cmin} , \frac{\cmin^2 D^4 |A|}{\min_{s \in S} \policytime{\fastpolicy}(s)^4} \log^4 \frac{D |S| |A|}{\delta \cmin} \biggr\}
    $. For smaller $K$,  we have 
    \[
        \regret 
        \le 
        \wt O \biggl( \frac{D^3 |S|^2 |A|}{\cmin^2} + \cmin^2 D^5 |A| + D^2 |S|^3 |A|^2 \biggr)
        \le
        \wt O \biggl( \frac{D^5 |S|^3 |A|^2}{\cmin^2} \biggr).
    \]
\end{theorem}

\begin{proof}
    First assume that $K$ is large enough.
    By union bounds, \cref{lem:wtD-upp-bound-on-D,lem:wt-D-not-big-1} and the regret bound of SSP-O-REPS3 all hold with probability at least $1 - 3 |S| |A| \delta$ (because of the $O(\cdot)$ notation it is the same as $1 - \delta$).
    Therefore, $\policytime{\fastpolicy}(s) \le \wt D(s) \le O(D)$ for all $s \in S$.
    During the first $L$ episodes our cost is bounded as follows,
    \begin{align*}
        \sum_{k=1}^L \sum_{i=1}^{\Ik} c_k(s_i^k,a_i^k)
        & \le
        LD + O \left( D |S| \sqrt{|A| L} \log \frac{L D |S| |A|}{\delta} + D^{3/2} |S|^2 |A| \log^2 \frac{L D |S| |A|}{\delta} \right)
        \\
        & \le
        O \left( \frac{D |S|}{\cmin} \sqrt{|A| K} \log \frac{K D |S| |A|}{\delta} + \frac{D^{3/2} |S|^2 |A|}{\cmin^2} \log^3 \frac{K D |S| |A|}{\delta} \right),
    \end{align*}
    and then we bound the regret as in \cref{thm:reg-bound-unknown-P} to get the final result (the extra regret that comes from enlarging the known state threshold is at most $L D |S| |A|$ which is of the same order).
    
    When $K$ is too small we might encounter an underestimate or an overestimate of some $\policytime{\fastpolicy}(s)$.
    In the rest of the proof assume that $K > 2400 |S|^2 |A| \log^2 \frac{D |S| |A|}{\delta \cmin}$ because otherwise we never go past the diameter estimation phase and the regret is bounded by
    \[
        \regret
        \le
        \wt O \biggl( D^{3/2} |S|^2 |A| \biggr).
    \]
    By following the proof of \cref{lem:wt-D-not-big-1}, for $K > 2400 |S|^2 |A| \cmin^2 \log^2 \frac{D |S| |A|}{\delta \cmin}$ we have that $\wt D \le O (D^{3/2})$.
    
    \paragraph{Underestimate.} 
    The problem with an underestimate is that now our regret bound does not hold against $\pi^\star$, but only against the best policy in hindsight $\pi^\star(\wt D(\sinit))$ with expected time of at most $\wt D(\sinit)$.
    In addition, we may loose $D$ every time we switch to the fast policy (and the reason was reaching a ``bad'' state) by the proof of \cref{lem:bern-first-reg-term-bound}.
    Thus, the regret bound of SSP-O-REPS3 (without diameter estimation) gives
    \begin{align}
        \label{eq:underestimate-1}
        \sum_{k=1}^K \sum_{i=1}^{\Ik} c_k(s_i^k,a_i^k) - \sum_{k=1}^K \ctg{\pi^\star(\wt D(\sinit))}_k(\sinit)
        \le
        O \biggl( \frac{D |S|}{\cmin} \sqrt{|A| K} \log \frac{K D |S| |A|}{\delta \cmin} + \frac{D^2 |S|^2 |A|}{\cmin^2} \log^2 \frac{K D |S| |A|}{\delta \cmin} + KD \biggr).
    \end{align}
    We can use this to bound the total cost, and therfore the regret of the learner, as follows
    \begin{align*}
        \regret
        & \le
        \sum_{k=1}^K \sum_{i=1}^{\Ik} c_k(s_i^k,a_i^k)
        \\
        & \le
        \sum_{k=1}^K \ctg{\pi^\star(\wt D(\sinit))}_k(\sinit) + O \biggl( \frac{D |S|}{\cmin} \sqrt{|A| K} \log \frac{K D |S| |A|}{\delta \cmin} + \frac{D^2 |S|^2 |A|}{\cmin^2} \log^2 \frac{K D |S| |A|}{\delta \cmin} + KD \biggr)
        \\
        & \le
        O \biggl( K \wt D(\sinit) + \frac{D |S|}{\cmin} \sqrt{|A| K} \log \frac{K D |S| |A|}{\delta \cmin} + \frac{D^2 |S|^2 |A|}{\cmin^2} \log^2 \frac{K D |S| |A|}{\delta \cmin} + KD \biggr)
        \\
        & \le
        O \biggl( K D + \frac{D |S|}{\cmin} \sqrt{|A| K} \log \frac{K D |S| |A|}{\delta \cmin} + \frac{D^2 |S|^2 |A|}{\cmin^2} \log^2 \frac{K D |S| |A|}{\delta \cmin} \biggr)
        \\
        & \le
        \wt O \biggl( D^5 |A| \cmin^2 + D^3 |S| |A| + \frac{D^2 |S|^2 |A|}{\cmin^2} \biggr),
    \end{align*}
    where the second inequality follows from \cref{eq:underestimate-1}, the third because the expected time of $\pi^\star(\wt D(\sinit))$ is at most $\wt D(\sinit)$, the forth because $\wt D(\sinit) \le D$ as an underestimate, and the last one is because underestimation may occur when $K < \frac{\cmin^2 D^4 |A|}{\min_{s \in S} \policytime{\fastpolicy}(s)^4} \log^4 \frac{D |S| |A|}{\delta \cmin}$ according to \cref{lem:wtD-upp-bound-on-D}.
    
    \paragraph{Overestimate.} 
    At this situation our regret bound holds, but its dependence in $\wt D(s)$ is problematic because $\wt D(s)$ overestimates $D$ for some $s \in S$.
    However, according to \cref{lem:wt-D-not-big-1}, this may occur only when $K < \cmin^2 D |S|^4 |A|^3 \log^2 \frac{D |S| |A|}{\delta \cmin}$.
    In addition, as mentioned before, $\wt D(s) \le O(D^{3/2})$.
    Thus, we have
    \begin{align*}
        \regret
        & \le
        O \biggl( \frac{\wt D(s) |S|}{\cmin} \sqrt{|A| K} \log \frac{K D |S| |A|}{\delta \cmin} + \frac{\wt D(s)^2 |S|^2 |A|}{\cmin^2} \log^2 \frac{K D |S| |A|}{\delta \cmin} \biggr)
        \\
        & \le
        O \biggl( \frac{D^{3/2} |S|}{\cmin} \sqrt{|A| K} \log \frac{K D |S| |A|}{\delta \cmin} + \frac{D^3 |S|^2 |A|}{\cmin^2} \log^2 \frac{K D |S| |A|}{\delta \cmin} \biggr)
        \\
        & \le
        \wt O \biggl( D^2 |S|^3 |A|^2 + \frac{D^3 |S|^2 |A|}{\cmin^2} \biggr).
    \end{align*}
\end{proof}

\newpage
\section{Zero costs}
\label{sec:zero-cost-proofs}

We can artificially fulfil \cref{ass:c-min} by adding a small $\epsilon > 0$ perturbation to the costs.
That is, when $c_k$ is revealed, we pass to the learner the perturbed cost function $\tilde c_k(s,a) = \max \{ c_k(s,a),\epsilon \}$ for every $s \in S$ and $a \in A$.

Notice that changing the cost function does not change the transition function or the SSP-diameter.
However, the bias introduced by our perturbation adds an additional $\epsilon D^\star K$ term to the regret, where $D^\star$ is the expected time it takes the best policy in hindsight to reach the goal state.

Choosing $\epsilon$ to balance the algorithms' regret with the new term yields the following regret bounds for the general case.
\cref{thm:exp-reg-full-info-general} matches \cref{thm:exp-reg-full-info}, \cref{thm:hp-reg-full-info-general} matches \cref{thm:hp-reg-full-info}, \cref{thm:reg-bound-unknown-P-general} matches \cref{thm:reg-bound-unknown-P}, and \cref{thm:reg-bound-unknown-P-unknown-D-best-general} matches \cref{thm:reg-bound-unknown-P-unknown-D}. 

\begin{theorem}
    \label{thm:exp-reg-full-info-general}
    Running SSP-O-REPS with known transition function, $\eta = \sqrt{\frac{3 \log (D |S| |A| / \epsilon)}{K}}$ and $\epsilon = K^{-1/4}$ ensures that
    \[
        \bbE [ \regret ]
        \le
        O \left( D^\star K^{3/4} \sqrt{\log (K D |S| |A|)} \right).
    \]
\end{theorem}

\begin{theorem}
    \label{thm:hp-reg-full-info-general}
    Running SSP-O-REPS2 with known transition function, $\eta = \sqrt{\frac{3 \log (D |S| |A| / \epsilon)}{K}}$ and $\epsilon = K^{-1/4} \sqrt{\log \frac{K D |S| |A|}{\delta}}$ ensures that, with probability $1 - \delta$,
    \[
        \regret
        \le
        O \left( D^\star K^{3/4} \log \frac{K D |S| |A|}{\delta} \right).
    \]
\end{theorem}

\begin{theorem}
    \label{thm:reg-bound-unknown-P-general}
    Running SSP-O-REPS3 with known SSP-diameter $D$, $\eta = \sqrt{\frac{3 \log (D |S| |A| / \epsilon)}{K}}$ and \newline $\epsilon=K^{-1/4} |S| \sqrt{|A| \log \frac{K D |S| |A|}{\delta}}$ ensures that, with probability $1 - \delta$,
    \[
        \regret
        \le
        O \left( D^\star |S| \sqrt{|A|} K^{3/4} \log \frac{K D |S| |A|}{\delta} + D^2 \sqrt{K} \log \frac{K D |S| |A|}{\delta} \right).
    \]
\end{theorem}

\begin{theorem}
    \label{thm:reg-bound-unknown-P-unknown-D-best-general}
    Running SSP-O-REPS3 with $\epsilon=K^{-1/4} |S| \sqrt{|A| \log \frac{K \wt D(\sinit) |S| |A|}{\delta}}$, $\eta = \sqrt{\frac{3 \log (\wt D(\sinit) |S| |A| / \epsilon)}{K}}$ and \newline $L = 2400 \max \{ |S|^2 |A| \log^2 \frac{K |S| |A|}{\delta \epsilon} , \frac{\sqrt{K}}{\epsilon \sqrt{|A|}} \log \frac{K |S| |A|}{\delta \epsilon} \}$ ensures that, with probability at least $1 - \delta$,
    \[
        \regret
        \le
        O \left( D^\star |S| \sqrt{|A|} K^{3/4} \log \frac{K D |S| |A|}{\delta} + D^2 \sqrt{K} \log^2 \frac{K D |S| |A|}{\delta} \right),
    \]
    for
    $
        K 
        \ge 
        \max \biggl\{ D^{2/3} |S|^4 |A|^{8/3} \log^2 \frac{D |S| |A|}{\delta} , \frac{D^{8/3} |S|^{4/3} |A|^{4/3} \log^{10/3} \frac{ D |S| |A|}{\delta}}{\min_{s \in S} \policytime{\fastpolicy}(s)^{8/3}} \biggr\}.
    $
    For smaller $K$, we have 
    \begin{align*}
        \regret 
        & \le 
        \wt O \biggl( D^\star |S| \sqrt{|A|} K^{3/4} + D^3 \sqrt{K} + \frac{D^5 |S|^2 |A|^2}{\sqrt{K}} + D^2 |S|^3 |A|^2 \biggr)
        \\
        & \le
        \wt O \biggl( D^\star |S| \sqrt{|A|} K^{3/4} + D^3 \sqrt{K} + D^5 |A| + D |S|^3 |A|^2 \biggr)
        \\
        & \le
        \wt O \biggl( D^\star |S| \sqrt{|A|} K^{3/4} + D^3 \sqrt{K} + D^5 |S|^3 |A|^2 \biggr).
    \end{align*}
\end{theorem}

Note that for $\epsilon \le 1$ in \cref{thm:reg-bound-unknown-P-general,thm:reg-bound-unknown-P-unknown-D-best-general}, we need $K \ge |S|^4 |A|^2$.
However, if $K < |S|^4 |A|^2$ (this is something the algorithm can check) we can just stay in the diameter estimation phase (i.e., assume all costs are $c(s,a)=1$) and get a regret of $\tO \bigl( D |S|^4 |A|^2 + D^{3/2} |S|^2 |A| \bigr)$ (or tune $\epsilon$ especially for this case for better results).

\newpage
\section{Concentration inequalities}
\label{sec:con-ineq}

\begin{theorem}[Anytime Azuma]
    \label{thm:anytime-azuma}
    Let $(X_n)_{n=1}^\infty$ be a martingale difference sequence such that $|X_n| \le B_n$ almost surely. 
    Then with probability at least $1 - \delta$,
    \[
        \Bigl| \sum_{n=1}^N X_n \Bigr| 
        \le 
        4 \sqrt{\sum_{n=1}^N B_n^2 \log \frac{N}{\delta}} \quad \forall N \ge 1.
    \]
\end{theorem}

\begin{lemma}[\cite{cohen2020ssp}, Lemma B.15] 
    \label{lem:martingalevariance}
    Let $(X_t)_{t=1}^\infty$ be a martingale difference sequence adapted to the filtration $(\calF_t)_{t=0}^\infty$. Let $Y_n = (\sum_{t=1}^n X_t)^2 - \sum_{t=1}^n \bbE[X_t^2 \mid \calF_{t-1}]$. Then $(Y_n)_{n=0}^\infty$ is a martingale, and in particular if $\tau$ is a stopping time such that $\tau \le c$ almost surely, then $\bbE[Y_\tau] = 0$.
\end{lemma}

\begin{lemma}[\cite{cohen2020ssp}, Lemma D.4] 
    \label{lem:martingalte-multiplicative-bound}
    Let $(X_n)_{n=1}^\infty$ be a sequence of random variables with expectation adapted to the filtration $(\calF_n)_{n=0}^\infty$.
    Suppose that $0 \le X_n \le B$ almost surely. Then with probability at least $1-\delta$, the following holds for all $n \ge 1$ simultaneously:
    \begin{equation}
        \label{eq:anytime-bern-3}
        \sum_{i=1}^n \bbE[X_i \mid \calF_{i-1} ]
        \le
        2 \sum_{i=1}^n X_i + 4 B \log \frac{2n}{\delta}.
    \end{equation}
\end{lemma}

\begin{lemma}
    \label{lem:bounded-exp-for-large-values}
    Let $X$ be a non-negative random variable such that $\Pr[|X| > m] \le a e^{-m/b}$ ($a \ge 1$) for all $m \ge 0$.
    Then, $\bbE [ X\indevent{X > r}] \le a(r+b)e^{-r/b}$.
\end{lemma}

\begin{proof}
    We have that,
    \[
        \bbE [ X\indevent{X > r}]
        =
        r \Pr [X > r] + \bbE [ (X-r) \indevent{X -r > 0}],
    \]
    and
    \begin{align*}
        \bbE [ (X-r) \indevent{X -r > 0}]
        & =
        \int_{m=0}^\infty \Pr [X-r > m] dm
        \\
        & =
        \int_{m=r}^\infty \Pr [X > m] dm
        \\
        & \le
        \int_{m=r}^\infty a e^{-m/b} dm
        \\
        & =
        ab e^{-r/b}.
    \end{align*}
    Hence $\bbE [ X\indevent{X > r}] \le a(r+b)e^{-r/b}$ as required.
\end{proof}

\begin{theorem}[Anytime Azuma for Unbounded Martingales]
    \label{thm:unbounded-azuma}
    Let $(X_n)_{n=1}^\infty$ be a non-negative martingale difference sequence adapted to the filtration $(\calF_n)_{n=1}^\infty$ such that $\Pr[|X_n| > m] \le a e^{-m/b}$ ($a \ge 1$) for all $n \ge 1$ and $m \ge 0$. 
    Then, with probability at least $1 - \delta$,
    \[
        \Bigl| \sum_{n=1}^N X_n \Bigr| 
        \le 
        11 b \sqrt{N \log^3 \frac{2aN}{\delta}} \quad \forall N \ge 1.
    \]
\end{theorem}

\begin{proof}
    Define $r_n = 2b \log \frac{2an}{\delta}$, and note that $\Pr [|X_n| > r_n] \le \frac{\delta}{4n^2}$.
    
    Additionally define $Y_n = X_n \indevent{|X_n| \le r_n} - \bbE \left[ X_n \indevent{|X_n| \le r_n} \mid \calF_{n-1} \right]$.
    $(Y_n)_{n=1}^\infty$ is a bounded martingale difference sequence, and by \cref{thm:anytime-azuma} we have that with probability at least $1 - \frac{\delta}{2}$,
    \[
        \Bigl| \sum_{n=1}^N Y_n \Bigr| 
        \le 
        4 \sqrt{\sum_{n=1}^N r_n^2 \log \frac{N}{\delta}}
        \quad \forall N \ge 1.
    \]
    
    Therefore, by a union bound, both the above holds and $|X_n| \le r_n$ for all $n \ge 1$ with probability at least $1 - \delta$. 
    We get that
    \[
        \Bigl| \sum_{n=1}^N X_n \indevent{|X_n| \le r_n} - \bbE \left[ X_n \indevent{|X_n| \le r_n} \mid \calF_{n-1} \right] \Bigr| 
        \le 
        4 \sqrt{\sum_{n=1}^N r_n^2 \log \frac{N}{\delta}},
    \]
    and simplifying using the definition of $r_n$ gets
    \[
        \Bigl| \sum_{n=1}^N X_n \indevent{|X_n| \le r_n} \Bigr| 
        \le
        \Bigl| \sum_{n=1}^N \bbE \left[ X_n \indevent{|X_n| \le r_n} \mid \calF_{n-1} \right] \Bigr|
        +
        8b \sqrt{N \log^3 \frac{2aN}{\delta}}.
    \]
    It thus remains to upper bound $\Bigl| \sum_{n=1}^N \bbE \left[ X_n \indevent{|X_n| \le r_n} \mid \calF_{n-1} \right] \Bigr|$. 
    First note that (since $X_n$ is a martingale difference sequence)
    \begin{align*}
        \bbE \left[ X_n \indevent{|X_n| \le r_n} \mid \calF_{n-1} \right]
        & =
        \bbE [X_n \mid \calF_{n-1} ] - \bbE \left[ X_n \indevent{|X_n| > r_n} \mid \calF_{n-1} \right]
        \\
        & =
        - \bbE \left[ X_n \indevent{|X_n| > r_n} \mid \calF_{n-1} \right],
    \end{align*}
    from which
    \begin{align*}
        \Bigl| \sum_{n=1}^N \bbE \left[ X_n \indevent{|X_n| \le r_n} \mid \calF_{n-1} \right] \Bigr|
        & =
        \Bigl| \sum_{n=1}^N \bbE \left[ X_n \indevent{|X_n| > r_n} \mid \calF_{n-1} \right] \Bigr|
        \\
        & \le
        \sum_{n=1}^N \bbE \Bigl[ | X_n | \indevent{|X_n| > r_n} \mid \calF_{n-1} \Bigr]
        \\
        & \le
        \sum_{n=1}^N a(r_n + b) e^{-r_n / b}
        \\
        & \le
        \sum_{n=1}^N 3ab \left( \frac{\delta}{2an} \right)^2 \log \frac{2an}{\delta}
        \\
        & \le
        \sum_{n=1}^N 6ab \left( \frac{\delta}{2an} \right)^2 \left( \frac{2an}{\delta} \right)^{1/2}
        \\
        & =
        \sum_{n=1}^N 6ab \left( \frac{\delta}{2an} \right)^{3/2}
        \\
        & \le
        \sum_{n=1}^N \frac{3b}{n^{3/2}}
        \le 3b \log (N+1) \le 3b \log (2N),
    \end{align*}
    where the second inequality follows from \cref{lem:bounded-exp-for-large-values} and and the forth inequality follows because $\log x \le 2 \sqrt{x}$.
\end{proof}